\pdfoutput=1
\documentclass[12pt]{article}
\usepackage[margin=1in]{geometry}
\usepackage{setspace}
\usepackage{enumitem}
\usepackage{amssymb}
\usepackage[]{natbib}
\usepackage{bbm}
\usepackage{amsmath,amsfonts,bm}
\allowdisplaybreaks
\usepackage[parfill]{parskip}
\usepackage{amsthm}
\usepackage{microtype}
\usepackage{siunitx}
\usepackage{graphicx}
\usepackage{booktabs} % for professional tables
\usepackage{multirow} 
\usepackage[toc,page]{appendix}
\usepackage{comment}
\usepackage{caption}
\usepackage{bbm}
\usepackage{subcaption}
\usepackage[dvipsnames]{xcolor}
\usepackage{pgfplots}
\pgfplotsset{compat=1.18}
\usepackage{color}
\usepackage{tikz}

\usetikzlibrary{shapes,decorations,arrows,calc,arrows.meta,fit,positioning}
\tikzset{
    -Latex,auto,node distance =1 cm and 1 cm,semithick,
    state/.style ={ellipse, draw, minimum width = 0.7 cm},
    point/.style = {circle, draw, inner sep=0.04cm,fill,node contents={}},
    bidirected/.style={Latex-Latex,dashed},
    el/.style = {inner sep=2pt, align=right, sloped}
}

\usepackage{algorithm}
\usepackage{chngcntr} % for figure labels in appendix
\usepackage[noend]{algpseudocode}
\algrenewcommand\textproc{}% Used to be \textsc
\algdef{SE}[SUBALG]{Indent}{EndIndent}{}{\algorithmicend\ }%
\algtext*{Indent}
\algtext*{EndIndent}

\newcommand{\p}{{\mathbb{P}}}

\newcommand{\B}{\mathcal{B}}

\newcommand{\prob}{{\mathbb{P}}}

\DeclareMathOperator*{\argmin}{arg\,min}

\DeclareMathOperator*{\argmax}{arg\,max}

\newtheorem{theorem}{Theorem}
\newtheorem{prop}{Proposition}
\newtheorem{lemma}{Lemma}

\newtheorem{remark}{Remark}
\newtheorem{assumption}{Assumption}[section]

\newtheorem{example}{Example}
\newtheorem{definition}{Definition}

\usepackage{dsfont}

\newcommand{\RNum}[1]{\uppercase\expandafter{\romannumeral #1\relax}}

\newcommand{\E}{\mathbb{E}}

\newcommand{\ZB}[1]{{\color{black}{ #1}}}

\usepackage{hyperref}

\hypersetup{
	unicode=false,
	pdftoolbar=true,
	pdfmenubar=true,
	pdffitwindow=false,     % window fit to page when opened
	pdfstartview={FitH},    % fits the width of the page to the window
	pdfkeywords={}, % list of keywords
	pdfnewwindow=true,      % links in new window
	colorlinks=true,       % false: boxed links; true: colored links
	linkcolor=cyan,          % color of internal links (change box color with linkbordercolor)
	citecolor=blue,        % color of links to bibliography
	filecolor=pink,      % color of file links
	urlcolor=cyan           % color of external links
}

\usepackage[utf8]{inputenc}

\newcommand{\neighbor}[1]%
{\overline{#1}}

\usepackage{mathtools}
\graphicspath{{fig/}}

\begin{document}
\title{ \bf A Tale of Two Cities: Pessimism and Opportunism in Offline Dynamic Pricing}
\date{}

\author{Zeyu Bian$^{1,2}$, Lan Wang$^{2}$, and Zhengling Qi$^{3}$ \\ 
$^1$Department of Statistics, Florida State University\\
$^2$Department of Management Science, University of Miami \\
$^3$ Department of Decision Sciences, The George Washington University}

\maketitle

\begin{abstract}
We study offline dynamic pricing when historical data provide incomplete coverage of the price space such that some candidate prices, including the optimal one, may be entirely unobserved.
This setting is common in practice and is especially difficult in dynamic environments. Existing offline reinforcement learning methods typically rely on full or partial coverage and can therefore perform poorly in such settings. We develop a nonparametric partial-identification framework for offline dynamic pricing that exploits the monotonicity of demand in price to bound the value of unobserved prices. Within this framework, we formulate two dynamic decision rules: a pessimistic policy that maximizes worst-case revenue and an opportunistic policy that minimizes worst-case regret. These rules are tailored to a sequential no-coverage environment and are not direct extensions of existing pessimistic offline RL or static opportunistic approaches. We establish finite-sample regret bounds for both policies, recovering the standard rate when the optimal price is covered and quantifying the additional cost when it is not. We also develop efficient algorithms and show, through simulations and an airline-ticket application, that our methods outperform standard offline RL baselines in no-coverage settings. Managerially, the framework provides a practical mapping from a firm’s risk posture to its pricing policy: firms seeking revenue stability and downside protection should prefer the pessimistic policy, whereas firms willing to bear measured risk for potential gains from under-explored prices should prefer the opportunistic policy.
\end{abstract}

\noindent\textbf{Keywords:} Offline dynamic pricing; Partial identification; Opportunism; Pessimism; Nonparametric demand model

\section{Introduction} \label{sec:intro}
Dynamic pricing \citep{gallego1994optimal,den2015dynamic} studies how firms adjust prices over time in response to demand and remaining inventory. Although the literature is vast, it focuses primarily on \textit{online} settings, where firms continue to experiment while making decisions. However, in many applications, such experimentation is costly, time-consuming, or infeasible. Poor prices can generate immediate revenue losses, distort demand, and impose high operational costs. Firms therefore often need to rely on historical transaction data when designing pricing policies. This makes offline dynamic pricing, which learns a pricing strategy from pre-collected data without further interaction, both practically important and fundamentally distinct from online learning \citep{levine2020offline}. 

\ZB{
Offline dynamic pricing is substantially more challenging than offline static pricing, recently studied by, for example, \citet{qi2022offline,bu2023offline}. In a static setting, the firm chooses a single price based on its immediate effect on revenue. In a dynamic problem, by contrast, current price affects not only current sales but also future states through inventory depletion, the remaining selling horizon, and future payoff. The object of interest is therefore not a one-shot revenue function, but a sequential decision problem. This dynamic structure makes the offline problem especially difficult: limited information about a price in one state can distort value estimation in later states and, in turn, the optimal policy itself.

Existing offline RL methods typically require the historical data to provide sufficient support and coverage over actions or prices \citep{antos2008learning,munos2008finite,farahmand2010error,yu2020mopo}. However, this requirement is often violated in practice, since historical pricing data are typically generated from a limited set of prices chosen for operational or business reasons, rather than through deliberate exploration of the full price space.

In this paper, we study offline pricing precisely in such no-coverage settings, where some candidate prices, including the optimal one, are entirely absent from the data. This lack of coverage poses a fundamental challenge and warrants special attention. When coverage fails, standard offline RL methods can yield suboptimal pricing decisions and, consequently, lower revenues. A direct application of existing offline RL algorithms is therefore inadequate for the problem we study.}

The no-coverage setting creates a serious challenge for both practice and theory. In practice, standard offline methods may disregard unobserved prices even when those prices are in fact optimal. In theory, once some prices are missing from the data, demand and Q-functions are no longer point identified, and the optimal dynamic pricing policy may fail to be identified as well. Recent pessimistic offline RL methods \citep{yu2020mopo,buckman2020importance,jin2021pessimism} relax full coverage by requiring only that the actions chosen by the optimal policy be represented in the offline data. In pricing applications, however, even this weaker condition is untestable ex ante and may easily be violated. Motivated by this gap, we study offline dynamic pricing \textit{without any coverage assumption}. Our goal is to develop a framework that remains valid in the empirically relevant setting where historical data are informative, but inherently incomplete.

Our approach exploits a fundamental economic regularity in pricing: all else equal, demand is nonincreasing in price. This monotonicity allows us to partially identify the demand function, and hence revenue at prices that are never observed in the historical data. Rather than extrapolating from a parametric demand specification, we construct interval bounds for the value of unobserved prices using neighboring observed prices. The decision problem is then no longer to maximize a point estimate of revenue, but to choose among prices whose values are only partially identified. This partial-identification perspective enables policy learning beyond the support of the observed data while maintaining a fully nonparametric demand model.

Within this framework, we develop two decision rules: a pessimistic rule and an opportunistic rule. The pessimistic policy chooses the price that maximizes worst-case revenue, offering protection against downside risk when uncertainty about unobserved prices is substantial. The opportunistic policy instead chooses the price that minimizes worst-case regret relative to the best feasible alternative, allowing the firm to pursue upside when a purely pessimistic rule would be too conservative.
These rules are not direct extensions of existing approaches. Prior pessimistic methods \citep{yu2020mopo,buckman2020importance,jin2021pessimism,uehara2022cppo} are typically developed under at least partial coverage, whereas prior opportunistic rules \citep{savage1951theory,manski2007minimax,stoye2009minimax} are mainly designed for static problems, where regret is defined relative to a one-period benchmark and uncertainty does not propagate dynamically. In our setting, by contrast, some prices may be entirely absent from the offline data, rendering both current and continuation revenues only partially identified. This uncertainty then propagates through the Bellman recursion. We therefore formulate dynamic versions of pessimism and opportunism tailored to offline pricing without coverage. We show that both policies admit strong theoretical guarantees in this challenging setting. Their formal definitions and analysis are developed in Sections~\ref{sec:pess2} and \ref{sec:mm}.
%Within this framework, we study two decision rules under partial identification: a pessimistic rule \citep{yu2020mopo,buckman2020importance} and an opportunistic rule \citep{savage1951theory,manski2007minimax,stoye2009minimax}. 
%However, adopting these ideas in our setting is not straightforward. 
%Existing pessimistic methods are typically developed for settings in which action values are point identified, while opportunistic methods are mainly designed for static environments.  

\ZB{\textbf{\textit{Managerial insights}}.
Our framework delivers a practical mapping from a firm’s risk posture to its pricing policy when historical data provide incomplete price coverage. Firms seeking revenue stability and protection against downside risk should prefer the pessimistic policy, whereas firms willing to bear measured risk for potential upside from under-explored prices should prefer the opportunistic policy. Beyond policy selection, the framework has two broader managerial uses. First, the estimated offline policy can serve as a warm start for subsequent online learning, reducing costly and risky experimentation at deployment. Second, it can serve as a benchmark for evaluating pricing rules already used in practice. By comparing the estimated offline policy with an incumbent heuristic, a firm can gauge whether the heuristic is likely to be suboptimal, and by how much, thereby informing whether to adopt a new policy or invest in additional data collection and model refinement.}

\subsection{Main Results and Contributions}

Offline dynamic pricing has received limited attention in the literature, despite its practical importance. Existing offline reinforcement learning methods, including standard dynamic programming \citep{bellman1957markovian} and recent pessimistic approaches \citep{yu2020mopo,buckman2020importance,jin2021pessimism}, typically rely on some form of coverage in the historical data. In pricing applications, however, such assumptions are often implausible because many candidate prices are only weakly observed or entirely unobserved. To the best of our knowledge, this is the first paper that develops a statistically sound offline dynamic pricing technique without data coverage assumption.

We make three major contributions. First, we develop an offline dynamic-pricing method that remains valid without any coverage assumption. Our approach leverages monotonicity in demand to construct a partial-identification framework for policy learning when some prices are missing from the offline data. This yields two decision rules, pessimistic and opportunistic, that are well defined even when the optimal price is not observed. To achieve this, we introduce a partial
identification framework for conducting pessimistic and opportunistic policy learning in the presence
of limited observed prices in offline data.  Importantly, our framework imposes no parametric restrictions on demand, the policy class, or the Q-function class.  In this sense, it provides a grounded solution to a central limitation of existing offline RL methods in pricing environments.

Second, we establish finite-sample regret bounds for both the pessimistic and opportunistic policies. The bounds decompose into two terms: one capturing estimation error for demand at observed prices, and the other capturing the additional cost induced by potentially unobserved optimal prices. When the optimal price is covered by the offline data, the second term vanishes and we recover the standard regret rate of $O_p(\sqrt{\log N / N})$, where 
$N$ is the number of trajectories. When coverage fails, our results quantify the precise cost of learning under partial identification; moreover, this cost is tight in the sense that it is attainable. More broadly, our analysis brings partial identification and opportunistic decision rules into a dynamic pricing problem, whereas most prior work in this vein has focused on static decision settings. Furthermore, our novel regret analysis for the opportunistic method distinguishes it from the existing literature, which focuses on static settings without dynamics, see, e.g., \citet{manski2007minimax,stoye2012minimax,Cui2021Individualized,christensen2022optimal,masten2023minimax,kido2023locally}. 

%Theoretically, we establish regret bounds for both strategies (Theorems \ref{thm: regret 2} and \ref{theorem:regret3}). Specifically, the regret bounds  consist of two components: 
%the first arises from the estimation of the unknown demand CDFs for the observed prices in the data, while the second reflects the error due to the potentially unobserved optimal price. 
%If the optimal policy is observed in the offline data, the second part is zero, and we recover the existing optimal regret bound, which is of the order $O_p\left(\sqrt{\log N/N}\right)$, where $N$ is the number of trajectories in the offline dataset. When the assumption on observing an optimal price in the offline data does not hold, our regret bound reveals the associated cost,  which is \textbf{\textit{tight}} in the sense that it is attainable. To our knowledge, this work represents the first endeavor to incorporate a partial identification framework and the opportunistic approach into dynamic pricing and sequential decision-making contexts. Furthermore, our novel regret analysis for the opportunistic method distinguishes it from the existing literature, which focuses on static settings without dynamics, see, e.g., \citet{manski2007minimax,stoye2012minimax,Cui2021Individualized,christensen2022optimal,masten2023minimax,kido2023locally}.  

Third, we develop efficient algorithms for both approaches and evaluate them in simulation and an airline-ticket application. The results show that our methods outperform existing offline RL baselines in no-coverage settings. Taken together, the paper connects offline dynamic pricing to partial identification in a dynamic environment and provides both methodological and practical tools for pricing with incomplete historical data.

\subsection{Related Literature}

This work relates to three strands of literature: offline reinforcement learning, dynamic pricing, and partial identification.

\noindent \textbf{\textit{Offline RL}}.
RL studies sequential decision-making problems in which the goal is to maximize long-run cumulative rewards \citep{sutton2018}.
Recent work has shown the promise of RL in a range of domains,
including dynamic pricing \citep{den2015dynamic}, artificial intelligence used in gaming \citep{silver2016mastering}, healthcare \citep{liao2022batch,shi2023value,bian2025off}, dynamic treatment regimens \citep{Murphy,gest,wallace2015doubly,shi2018high,bian2023variable}, and robotics \citep{levine2020offline}. 

In offline RL, however, the learner must rely on historical data collected under a fixed behavior policy, without further interaction with the environment, often resulting in datasets that do not fully cover all possible action-state pairs. Classical offline RL methods address this difficulty by imposing a full coverage assumption, requiring the historical data to sufficiently represent all candidate actions \citep{antos2008learning,munos2008finite,farahmand2010error}.
 More recent studies \citep{fujimoto2019off,yu2020mopo,kumar2020conservative, buckman2020importance,jin2021pessimism, uehara2022cppo} show that by incorporating the pessimistic principle, the full coverage assumption can be relaxed to the weaker requirement that the actions chosen by the optimal policy are covered. \citet{chen2023steel} go further and eliminate explicit coverage conditions by exploiting Lebesgue’s decomposition theorem, but their approach still relies on structural assumptions on the Q-function that justify extrapolation.

Our paper differs from this literature in two ways. First, we study a pricing environment in which the optimal price may be entirely unobserved in the data. Second, unlike \citet{jin2021pessimism} and \citet{chen2023steel}, we impose no parametric restrictions on demand, the policy class, or the Q-function class. Our pessimistic policy is therefore not a direct application of existing pessimistic RL methods: it is designed to account explicitly for unobserved prices by exploiting demand monotonicity. As we show in Section \ref{sec:sims}, state-of-the-art RL methods, including the pessimistic approach \citep{jin2021pessimism}, Conservative Q-Learning \citep{kumar2020conservative}, and Batch-Constrained Q-learning (\citep{fujimoto2019off}), perform suboptimally in this setting and are outperformed by our proposed methods.

\noindent \textbf{\textit{Dynamic pricing}}. Existing literature on dynamic pricing mainly focuses on the online setting \citep{broder2012dynamic,keskin2014dynamic,javanmard2019dynamic,nambiar2019dynamic,ban2021personalized,ma2021dynamic,elmachtoub2021power,bastani2022meta,luo2024distribution,liu2025fairness}. For a comprehensive review, see \citet{bitran2003overview,den2015dynamic}, and the references therein. By contrast, work on offline dynamic decision-making is much more limited. \citet{levi2007provably} and \citet{ban2020confidence}
study offline inventory control, whereas our focus is offline dynamic pricing.

\ZB{Among the papers closest to ours,
\citet{qi2022offline} study offline pricing with censored demand from a causal inference perspective and obtain identification
by imposing a full-coverage assumption. 
\citet{bu2023offline} study static offline pricing with censored demand under a linear model and emphasize identification through a distributionally robust optimization framework. 
Our paper differs from both lines of work, as they require a coverage assumption. In contrast, we study a more challenging and empirically relevant setting in which no coverage assumption is imposed. Moreover, we study a dynamic pricing problem rather than a static one, allow candidate prices to be entirely unobserved in the historical data, and impose no parametric structure on demand. In our setting, limited price coverage, rather than censoring, is the central challenge. We address this through a partial-identification framework that exploits demand monotonicity to bound values at unobserved prices and to construct pessimistic and opportunistic pricing policies with regret guarantees. 
}

\noindent \textit{\textbf{Partial identification}}.
Our paper is also related to the literature on partial identification in statistics and economics  \citep[e.g.,][]{manski2007minimax,stoye2012minimax,Cui2021Individualized,christensen2022optimal,masten2023minimax,kido2023locally}.
This literature differs from our work in two important respects. First, it focuses almost exclusively on static decision problems, whereas we study a dynamic setting in which uncertainty propagates through the Bellman recursion. Our regret analysis for the opportunistic policy is new in this sequential environment. Second, the source of non-identification is different. Prior work often considers non-identification arising from missing counterfactual outcomes or unmeasured confounding \citep[see, e.g.,][]{Cui2021Individualized}. In our setting, non-identification arises because some prices are unobserved in the offline data. Our paper therefore connects partial identification to offline dynamic pricing in a sequential, no-coverage environment.

The rest of this paper is organized as follows. In Section \ref{sec:model}, we introduce the offline dynamic pricing model, and discuss the limitations of the existing standard dynamic programming and  pessimistic approaches.
In Section \ref{sec:partial}, we describe how to frame our pricing problem as a partial identification task. In Section \ref{sec:algorithms}, we present our proposed pessimistic and opportunistic strategies; we also establish the regret guarantees for both methods in this section. We conduct numerical studies and a real-data analysis of airline ticket purchases in Section~\ref{sec:sims} to demonstrate the proposed methodologies. Finally, we conclude with a discussion in Section~\ref{sec:discussion}.

\section{Models and  Assumptions} \label{sec:model}

We begin by providing an overview of dynamic pricing and our model formulation in Section \ref{sec: background}. Then we briefly introduce two standard baseline methods for estimating the optimal policy and discuss their limitations in Section \ref{sec:standard}.

\subsection{Background} \label{sec: background}

Consider a decision process $\{ (X_{t}, A_{t}, D_{t}): 1 \le t \le T \}$ with $T$ time points. At each period $t$, $X_{t} \in \mathcal{X} \equiv \{0, 1, 2, 3, \dots, L\}$ denotes the inventory level, where $L$ is a fixed positive integer, $A_{t} \in \mathcal{A} \equiv \{a_1, a_2, \dots, a_K\}$ is the price, and $D_{t}$ represents the demand. A {\textit{policy}} prescribes a pricing strategy that defines the action $A_t$ an agent takes based on the history information $X_t \cup \{ (X_{k}, A_{k}, D_{k}): 1 \le k \le t-1 \} $, and it is a function that maps the history to a probability distribution over prices (actions). Denote $\pi_t$ as any feasible policy at time point $t$, and let $\pi = \{\pi_t\}_{1\leq t\leq T}$ denote a sequence of policies from time $1$ to $T$. The aim of dynamic pricing is to find an optimal policy denoted as $\pi^* \equiv \{\pi^*_t\}_{1\leq t\leq T}$ such that
\begin{gather*}
        \pi^* \in \argmax_{\pi }\E^{\pi}\left[\sum_{t=1}^TR_t\right],
    \end{gather*} where $R_t=\min(D_{t},X_{t})\times A_{t}$ is the reward at time $t$. Here, $\E^\pi$ denotes the expectation under the assumption that the system dynamics follow the policy $\pi$. At each time step $t$, the action $A_t$ is drawn from the policy $\pi_t$, with the transition dynamics governed by the probability distribution $\Pr(X_{t+1}, R_{t}|A_t, X_t, \{X_j,A_j,R_j\}_{1\le j<t})$.

Throughout the paper, we focus on the offline finite horizon RL setting, in which the data consists of $N$ independent trajectories, each with $T$ time periods, and can be summarized as  $\mathcal{D}_N = \{ (X_{i,t}, A_{i,t}, D_{i,t}): 1\le i\le N, 1\le t\le T\}$. Denote the policy used to generate the offline data as $\pi^b$, also known as the behavior policy, i.e., $A_{i,t} \sim \pi^b_t$, which could be history-dependent. Our goal is to leverage the pre-collected offline data $\mathcal D_N$ to estimate the optimal pricing strategy $\pi^\ast$. Next, we impose a modeling assumption on the decision process $\{ (X_{t}, A_{t}, D_{t}): 1 \le t \le T \}$.
%\LW{ In the MDP formulation, we need to clarify the state space which does not include $D_{i,t}$. Note that in stating the transition probabilities, the demand is not part of the state space, rather it is part of the random reward.}

\begin{assumption}[Demand Model] \label{assumption:poisson}
    (i) Given the price $A_t$, the demand $D_t$ is independent of $X_t$ and past history $ \{ (X_{k}, A_{k}, D_{k}): 1 \le k \le t-1 \} $. \\
    (ii) The inventory level at $t+1$ is $X_{t+1}=X_{t}-\min(X_{t},D_{t})$, for all $t$, i.e., there is no inventory replenishment.
\end{assumption}

Given Assumption \ref{assumption:poisson}, the Markov assumption \citep{puterman2014markov, sutton2018} automatically holds in our setting. Specifically, for all $t$, \begin{eqnarray*}
    \Pr(X_{t+1}, R_{t}|A_t, X_t, \{X_j,A_j,R_j\}_{1\le j<t})=\mathbb P_t(X_{t+1}, R_{t}|A_{t}, X_{t}),
\end{eqnarray*} where $\mathbb P_t(\cdot, \cdot\;|\cdot,\;\cdot)$
denotes the transition probability function at time $t$. The Markov assumption \citep{puterman2014markov, sutton2018} is frequently imposed in the RL literature, which is the foundation of many state-of-the-art algorithms \citep{fujimoto2019off,levine2020offline,kumar2020conservative}.

By the Markov property, it suffices to focus on finding an optimal Markov policy $\pi^\ast$ that assigns the price based only on 
$X_t$, the current inventory level at time $t$ \citep{puterman2014markov}. To obtain $\pi^\ast$, a popular approach is to compute the optimal Q-function. The optimal Q-function (also called action-value function) with inventory level $x$
and price $a$ at time $t$ is defined as 
$$Q_t^*(x,a)\equiv \max_\pi \E^{\pi}\left[\sum_{k=t}^T R_k|X_t=x,A_t=a\right].
$$ 
In addition, the optimal value function at time $t$ with an inventory level $X_t = x$ is defined as
$$V_t^*(x) \equiv \, \max_{a \in \mathcal A} Q_t^*(x,a).$$ Then under Assumption \ref{assumption:poisson}, an optimal policy $\pi^\ast$ at time $t$ can be obtained via \begin{gather*} 
        \pi_t^*(a|x) = \begin{cases} 
1, & \text{if } a = a^*_t(x), \, \text{where} \, a^*_t(x)= \argmax_{a\in \mathcal{ A}} \, Q_t^*(x,a) \\
0, & \text{otherwise},
\end{cases}
\end{gather*} 
%where $a^*_t(x)$ is the associated optimal decision at time $t$ defined as
%\begin{gather*}
%    a^*_t(x)= \argmax_{a\in \mathcal{ A}}Q_t^*(x,a),
%\end{gather*}
which is \textbf{\textit{deterministic}} and \textbf{\textit{greedy}} \citep{puterman2014markov} with respect to the corresponding optimal Q-function. For simplicity, here we assume the uniqueness of maximizing the optimal Q-function and thus focus on the deterministic policy class.

\ZB{\begin{remark}
    The uniqueness assumption is made mainly for technical simplicity. If the optimal Q-function is not unique, our results can still be interpreted with respect to a deterministic policy defined via $\operatorname{sargmax}$, i.e., choosing the smallest action among all maximizers of the Q-function. Formally, the optimal policy can be written as
$
\pi^{\mathrm{opt}}(a|x)=
\begin{cases}
1, & \text{if } a = \operatorname{sargmax}_{a' \in \mathcal A} Q^{\mathrm{opt}}(x,a'),\\[4pt]
0, & \text{otherwise}.
\end{cases}
$
Such a rule preserves optimality while yielding a single deterministic decision for each state. In our pricing setting, this refinement is economically meaningful: among revenue-equivalent prices, selecting the lower price typically leads to a higher conversion rate. A larger customer base is valuable for long-term brand loyalty and data collection.
\end{remark}
}
    
To compute $Q^\ast_t$, existing model-free RL algorithms often rely on the Bellman optimality equation given below:
\begin{gather} \label{eq: bellman}
        Q_t^*(x,a)= \E \left[  R_t+ V_{t+1}^*(X_{t+1})  |X_t=x,A_t=a \right],
\end{gather}
which can be regarded as a nonlinear conditional moment restriction \citep[e.g.,][]{chen2012estimation}. Moreover, for notation convenience, we define the Bellman operator $\B_t$  at time $t$ as  \begin{gather*}
    (\B_t g)(x,a)=\E \left[  R_t+ g(X_{t+1})   |X_t=x,A_t=a \right],
\end{gather*} 
for any measurable function $g: \mathcal X \rightarrow \mathbb{R}$, where $\E$ is taken with respect to the randomness in the immediate reward and the next state.
Thus, the Bellman optimality equation can be represented in a compact form as  $Q_t^*(x,a)=(\B_t V^*_{t+1})(x,a),$ for $t \geq 1$. %The Bellman optimality equation provides the foundation for many existing RL algorithms. 

To conclude this subsection, we now provide the Bellman optimality equation for our specific pricing problem. Throughout the paper, we assume the following boundary condition, which states that $V^\ast_{T+1}(x)=0$ for any $x$ and $Q^\ast_t(0,a)$ is $0$ for any $a$. The following proposition introduces the basic building block of our estimation procedure. Its proof can be found in Section \ref{appendix:prop} of the Appendix. 

\ZB{\begin{prop}\label{prop: bellman}
    Under Assumption \ref{assumption:poisson}, Equation \eqref{eq: bellman} becomes 
    \begin{align*}
        Q_t^*(x,a)=\underbrace{a\left[x-\sum_{d=0}^{x-1}F_t(d|a)\right]}_{\mbox{expected immediate reward}} +\underbrace{\sum_{d=0}^{x-1} \left[V_{t+1}^*(x-d)-V_{t+1}^*(x-d-1)\right] F_t(d|a)}_{\mbox{expected future cumulative reward}},
    \end{align*} where $F_t(d|a)\equiv \Pr(D_t \leq d |A=a)$ is the conditional cumulative distribution function (CDF) of  demand at time $t$ given price $a$. This expression holds for all $t$, $x,$ and $a$.
\end{prop}
}

In the following subsection, we discuss two widely used methods for estimating $Q^\ast_t$ and $\pi^\ast_t$.

\subsection{Baseline Methods: Greedy and Pessimistic Approaches
} \label{sec:standard}

In the following, we briefly introduce two frequently employed standard baseline methods for estimating the optimal policy and discuss their respective limitations.

\ZB{
\textbf{\textit{Greedy approach.}} Recall that the Q-function at each time point can be obtained using the Bellman equation presented in Proposition \ref{prop: bellman}. To estimate $\mathcal B_t$, it is sufficient to estimate the CDF $F_t(d|a)$ given as below:
$$\widehat F_t(d|a)=\frac{1}{N_t(a)}\sum_{i=1}^N \mathds 1(D_{i,t}\leq d, A_{i,t}=a),$$ for all $a \in \mathcal{A}_t^{\mathcal D_N}$, where $\mathcal{A}_t^{\mathcal D_N}=\{a\in \mathcal{A}: N_t(a)>0\}$ denotes the set of prices observed in the offline data at time $t$, and $N_t(a)=\sum_{i=1}^N \mathds{1}(A_{i,t}=a)$. Then one could plug $\widehat F_t(d|a)$ into the Bellman equation in Proposition \ref{prop: bellman} to sequentially estimate the optimal policy in a backward manner from $t = T$ to $t = 1$. Let $\widehat V^{\mathbf{gree}}_t(x)$ and $\widehat Q_t^{\mathbf{gree}}(x,a)$ represent the estimated value function and Q-function at each decision point $t$, respectively, under the estimated greedy policy. The estimated Q-function at time $t$ is obtained by applying the estimated Bellman operator $\widehat \B_t$ to the estimated value function at time $t+1$, i.e.,
 $\widehat Q_t^{\mathbf{gree}}(x,a)=(\widehat \B_t \widehat V_{t+1}^{\mathbf{gree}})(x,a)$. 
The estimated policy is then derived by selecting the greedy action with respect to the estimated Q-function. We summarize the greedy strategy in the following Algorithm \ref{alg:greedy}. }

\begin{algorithm}[t]
	\caption{Pseudocodes for Estimating 
  $\widehat\pi^{\mathbf{gree}}$: A Greedy Approach}\label{alg:greedy}
	\begin{algorithmic}[1]
		\small
		\Function{}{$ \mathcal D_N$}
 \State Compute $\widehat F_t(d|a)$ for all $t$, $d\leq L$, and $a \in \mathcal{A}_t^{\mathcal D_N}$.
  
        \State Set $\widehat V_{T+1}^{\mathbf{gree}} \gets 0$.
        
    \For {$t=T, T-1, \dots , 1$}

        \State $\widehat Q_{t}^{\mathbf{gree}}(x,a)=\begin{cases}
			0, & \text{if $x=0$}\\
   (\widehat\B_{t} \widehat V_{t+1}^{\mathbf{gree}}) (x,a),& \text{otherwise}
		 \end{cases}$.
        
        \State Compute $\widehat a_{t}^{\mathbf{gree}}(x)=\argmax_{a \in \mathcal{A}_t^{\mathcal D_N}} \widehat Q_{t}^{\mathbf{gree}}(x,a)$, and $\widehat V_{t}^{\mathbf{gree}}(x)= \widehat Q_{t}^{\mathbf{gree}}(x,\widehat a_{t}^{\mathbf{gree}}(x))$.

        \State Compute $\widehat \pi_t^{\mathbf{gree}}(a|x) = \begin{cases} 
1, & \text{if } a = \widehat a_{t}^{\mathbf{gree}}(x), \\
0, & \text{otherwise}. \end{cases}$
\EndFor

\State \Return $\widehat\pi^{\mathbf{gree}}=\{\widehat \pi^{\mathbf{gree}}_t \}_{1\leq t\leq T}$.
		\EndFunction
\end{algorithmic}

\end{algorithm}

% Then one can estimate the optimal Q-function
% $\widehat Q^{\mathbf{gree}}_t(x,a)$ via
% \begin{align} \label{eq:q-function}
%         \sum_{d=0}^x e^{-\widehatF_t(d|a)} \frac{\widehatF_t(d|a)^d}{d\,!}  da +ax\sum_{d=x+1}^\infty e^{-\widehatF_t(d|a)} \frac{\widehatF_t(d|a)^d}{d\,!} +\sum_{d=0}^{x-1} \widehat V_{t+1}^{\mathbf{gree}}(x-d) e^{-\widehat\lambda_{t}(a)} \frac{\widehat\lambda_{t}(a)^d}{d\,!} ,
%     \end{align}  

% Given the estimated greedy Q-function at time $t$, the estimated greedy decision is $\widehat a_{t}^{\mathbf{gree}}(x)=\argmax_a \widehat Q_{t}^{\mathbf{gree}}(x,a)$, and $\widehat V_{t}^{\mathbf{gree}}(x)= \widehat Q_{t}^{\mathbf{gree}}(x,\widehat a^{\mathbf{gree}}_{t}(x))$. Then at time $t-1$, $ \widehat Q_{t-1}^{\mathbf{gree}}(x,a)$ is estimated using Bellman equation in \eqref{eq:q-function}, i.e., $ \widehat Q_{t-1}^{\mathbf{gree}}(x,a)=(\widehat \B_{t-1} \widehat V_{t}^{\mathbf{gree}})(x,a),$ and the estimation process proceeds iteratively by applying the estimated Bellman operator recursively. 

\textbf{\textit{Vanilla pessimistic approach}}.
Similarly to the greedy approach, the pessimistic approach recursively applies the estimated Bellman operator to the estimated value function to compute the estimated Q-function, from $t =T$ to $t=1$.  For all $t$, denote $\widehat Q_{t}^{\mathbf{vp}}(x,a)$ and $\widehat V_{t}^{\mathbf{vp}}(x)$ as the estimated pessimistic Q-function and value function, respectively. Here, ``$\mathbf{vp}$'' stands for the \textbf{\textit{vanilla pessimistic}} approach.
However, instead of directly maximizing the estimated Q-function with respect to the action, as in the greedy approach, the pessimistic strategy maximizes a penalized version of the estimated Q-function: \begin{gather*}
    \widehat a^{\mathbf{vp}}_t(x)= \argmax_{a\in \mathcal{ A}_t^{\mathcal D_N}} \widehat Q_{t}^{\mathbf{vp}}(x,a),  \quad 
    \mbox{ where } \widehat Q_{t}^{\mathbf{vp}}(x,a)= (\widehat\B_{t} \widehat V_{t+1}^{\mathbf{vp}}) (x,a)-\delta_t(a).
\end{gather*} Here, the term $\delta_t(a)$ characterizes the uncertainty  of estimating the Q-function corresponding to action $a$,
and its specific form will be introduced later.
%, with the specific value of $\delta_t(a)$ to be discussed later. 
 Greater uncertainty associated with action $a$  yields  a higher value of $\delta_t(a)$, resulting in a larger penalization, and making the action less favorable for decision-making. The intuition of pessimistic approach is to discourage actions that are less explored in the offline dataset, preventing over-exploration.  We summarize the detailed estimation procedure for the vanilla pessimistic method in Algorithm \ref{alg:babypess}.

\begin{algorithm}[H]
	\caption{Pseudocodes for Estimating the Vanilla Pessimistic Policy
  $\widehat\pi^{\mathbf{vp}}$.}\label{alg:babypess}
	\begin{algorithmic}[1]
		\small
            
		\Function{}{$ \mathcal{ D}_N$}

  \State Compute $\widehat F_t(d|a)$ and $\delta_t(a)$, for all $t$, $d\leq L$, and $a \in \mathcal{A}_t^{\mathcal D_N}$.
    
        \State Set $\widehat V_{T+1}^{\mathbf{vp}} \gets 0$.

        \For {$t=T, T-1, \dots , 1$}

        \State $\widehat Q_{t}^{\mathbf{vp}}(x,a)=\begin{cases}
			0, & \text{if $x=0$}\\
   (\widehat\B_{t} \widehat V_{t+1}^{\mathbf{vp}}) (x,a)-\delta_t(a),& \text{otherwise}
		 \end{cases}$.
        
        \State Compute $\widehat a_{t}^{\mathbf{vp}}(x)=\argmax_{a \in \mathcal{A}^{\mathcal{ D}_N}_t} \widehat Q_{t}^{\mathbf{vp}}(x,a)$, and $\widehat V_{t}^{\mathbf{vp}}(x)= \widehat Q_{t}^{\mathbf{vp}}(x,\widehat a_{t}^{\mathbf{vp}}(x))$.

        \State Compute $\widehat \pi_t^{\mathbf{vp}}(a|x) = \begin{cases} 
1, & \text{if } a = \widehat a_{t}^{\mathbf{vp}}(x), \\
0, & \text{otherwise}. \end{cases}$
\EndFor

\State \Return $\widehat\pi^{\mathbf{vp}}=\{\widehat \pi^{\mathbf{vp}}_t \}_{1\leq t\leq T}$.

		\EndFunction
\end{algorithmic}

\end{algorithm}

\subsection{Limitations of Baseline Methods
}\label{sec:limitation}

Before discussing the limitations of the two baseline methods, we first introduce several key concepts: the regret of a policy, the full coverage assumption, and the partial coverage assumption.

To assess the quality of an estimated policy $\widehat \pi$,  we  use the regret $\mu^{\widehat\pi}$ defined as \begin{align*}
    \mu^{\widehat\pi} \equiv \E\left[V_1^*(X_1)-V_1^{\widehat \pi}(X_1)\right],
\end{align*} where the probability measure for $X_1$ is the same as our offline data distribution. By definition, the regret is always non-negative and a smaller regret indicates  better performance for policy $\widehat \pi$.

\textbf{\textit{Full coverage}}.
Denote $\prob_t^{\pi}(x,a)$ as the marginal distribution of the inventory level and price observed at time $t$, following the policy $\pi$,
and $\prob_t^{\pi^b}(x,a)$ is similarly defined for the behavior policy. The full coverage assumption holds if there exists a universal constant $C$, such that 
\begin{gather*} 
   \sup_{\pi}\max_t \max_{x,a} \frac{\prob_t^{\pi}(x,a)}{\prob_t^{\pi^b}(x,a)} \leq C.
\end{gather*}

This full coverage assumption, also known as the concentrability assumption \citep{antos2008learning,munos2008finite,farahmand2010error}, requires that the offline data-generating process sufficiently covers the whole range of prices that will be deployed by the company.

\textbf{\textit{Partial coverage}}. The partial coverage holds if there exists a universal constant $C$, such that
\begin{gather} \label{eq:partial}
  \max_t \max_{x,a}  \frac{\prob_t^{\pi^*}(x,a)}{\prob_t^{\pi^b}(x,a)} \leq C. 
\end{gather}

In contrast to the full coverage assumption, the partial coverage assumption only requires the probability ratio between the optimal policy and the behavior policy to be upper bounded, rather than for any policy $\pi$. Therefore, the partial coverage assumption is much weaker.

Under the full coverage assumption and Assumption \ref{assumption:poisson} in our pricing setting, it can be shown that the regret for the greedy approach outlined in Algorithm \ref{alg:greedy}  is roughly upper bounded by \begin{gather*}
 \sup_{\pi \in \Pi} \E^{\pi^b}\left[ \sum_{t=1}^T \frac{\prob_t^{\pi}(X_t,A_t)}{\prob_t^{\pi^b}(X_t,A_t)} \sqrt{\log N /N_t(A_t)} \right]+\E^{\pi^b}\left[ \sum_{t=1}^T \frac{\prob_t^{\pi^*}(X_t, A_t)}{\prob_t^{\pi^b}(X_t,A_t)} \sqrt{\log N /N_t(A_t)} \right],
\end{gather*}  where $\Pi=\{\pi: \forall x, a \in \mathcal{A}, \;   \prob_t^{\pi}(x,a)>0 \implies \prob_t^{\pi^b}(x,a)>0 \}$.
Thus, the regret of the greedy method is of order $O_p(\sqrt{\log N/N})$  and vanishes asymptotically. Nonetheless, this full coverage assumption is  unlikely to hold for the offline data in pricing problem, since the seller would not set an unreasonable price whose expected revenues are obviously suboptimal.

To relax the stringent full coverage assumption, researchers have recently employed the pessimistic strategy \citep{yu2020mopo,buckman2020importance,jin2021pessimism}. It can be shown that under Assumption \ref{assumption:poisson} together with the partial coverage assumption, the regret of the pessimistic approach outlined in Algorithm \ref{alg:babypess} is upper bounded by \begin{gather*}
\E^{\pi^b}\left[ \sum_{t=1}^T \frac{\prob_t^{\pi^*}(X_t, A_t)}{\prob_t^{\pi^b}(X_t,A_t)} \sqrt{\log N /N_t(A_t)} \right],
\end{gather*} up to a constant. Therefore, the regret is again of order $O_p(\sqrt{\log N/N})$  and vanishes asymptotically. Although the pessimistic approach relaxes the full coverage assumption required by the greedy approach to the partial coverage assumption, such an assumption can still be easily violated because it is rarely the case that the company implemented the optimal pricing strategy in the past/offline data. Indeed, if a specific price $a$ is missing, then $\widehat F_t(d|a)$ as well as $\delta_t(a)$ cannot be constructed, and both greedy and pessimistic methods would not take the action $a$ into account, leading to sub-optimal decision-making when the unobserved price $a$ coincides with the optimal pricing.

In the next section, we address such a challenge using a partial identification framework by leveraging the monotonic effect of the price on demand.

\section{A Partial Identification Framework} \label{sec:partial}

%In the previous section, we have discussed the limitation of the standard greedy and pessimistic methods. 

In this section, we examine the challenging scenario in which no coverage assumption is imposed; that is, we allow any price, including the optimal one, to be unobserved in the offline dataset. In this scenario, both
the standard greedy and pessimistic methods
may fail. To address this challenging setting, we consider a partial identification framework.

%\textbf{\textit{A partial identification framework}}. 
\subsection{Motivations}
By leveraging the inherent monotonicity property present in most pricing problems, we can construct a partial-identification bound for the demand model even if its associated price is unobserved. This bound enables us to seek an optimal pricing strategy beyond the observed prices in the offline data.

Intuitively, if a price is never observed in the historical data, the corresponding demand distribution cannot be point-identified. The monotonicity assumption still gives usable information: if higher prices do not increase demand, then the demand at an unobserved middle price must lie between the demands at nearby lower and higher observed prices. Therefore, instead of estimating a single value, we estimate an interval (a partial-identification bound). Policy learning is then performed over these intervals.  In particular, the mean demand $\E(D \mid a)$ at the unobserved price $a$ is bounded by the mean demands at the observed prices $a^{-}$ and $a^{+}$:
\[
\E(D \mid a) \in \left[\E(D \mid a^{+}),\, \E(D \mid a^{-})\right].
\]

Since demand is only partially identifiable for unobserved prices, expected revenue is also only partially identifiable: for each action, the decision maker has an interval, not a point estimate. The practical question is then no longer ``which action has the largest estimated value,'' but ``how should one choose when each action has a range of plausible values?''

To make this concrete, consider a two-armed bandit with actions $a=0$ and $a=1$, and let $Q_0$ and $Q_1$ denote their expected revenues. If either action, interpreted as a price, is never observed in the historical data, then the corresponding demand distribution cannot be identified, and consequently neither can its expected revenue. However, monotonicity and neighboring observed prices still allow us to build bounds $Q_a \in [Q_a^L, Q_a^U]$ for $a \in \{0,1\}$. Figure~\ref{fig:pess vs oppor} visualizes this interval-valued decision problem. Given Figure~\ref{fig:pess vs oppor}, two principled choices naturally arise. A conservative decision maker emphasizes downside protection and compares lower bounds (i.e., $Q_1^L$ versus $Q_0^L$). A less conservative decision maker emphasizes opportunity cost and compares worst-case regret (i.e., $Q_0^U - Q_1^L$ versus $Q_1^U - Q_0^L$) across actions. These correspond to the pessimistic and opportunistic criteria, respectively, which we discuss in detail later.

\begin{figure}[H]
    \centering
\includegraphics[height=6cm,
width=0.8\textwidth]{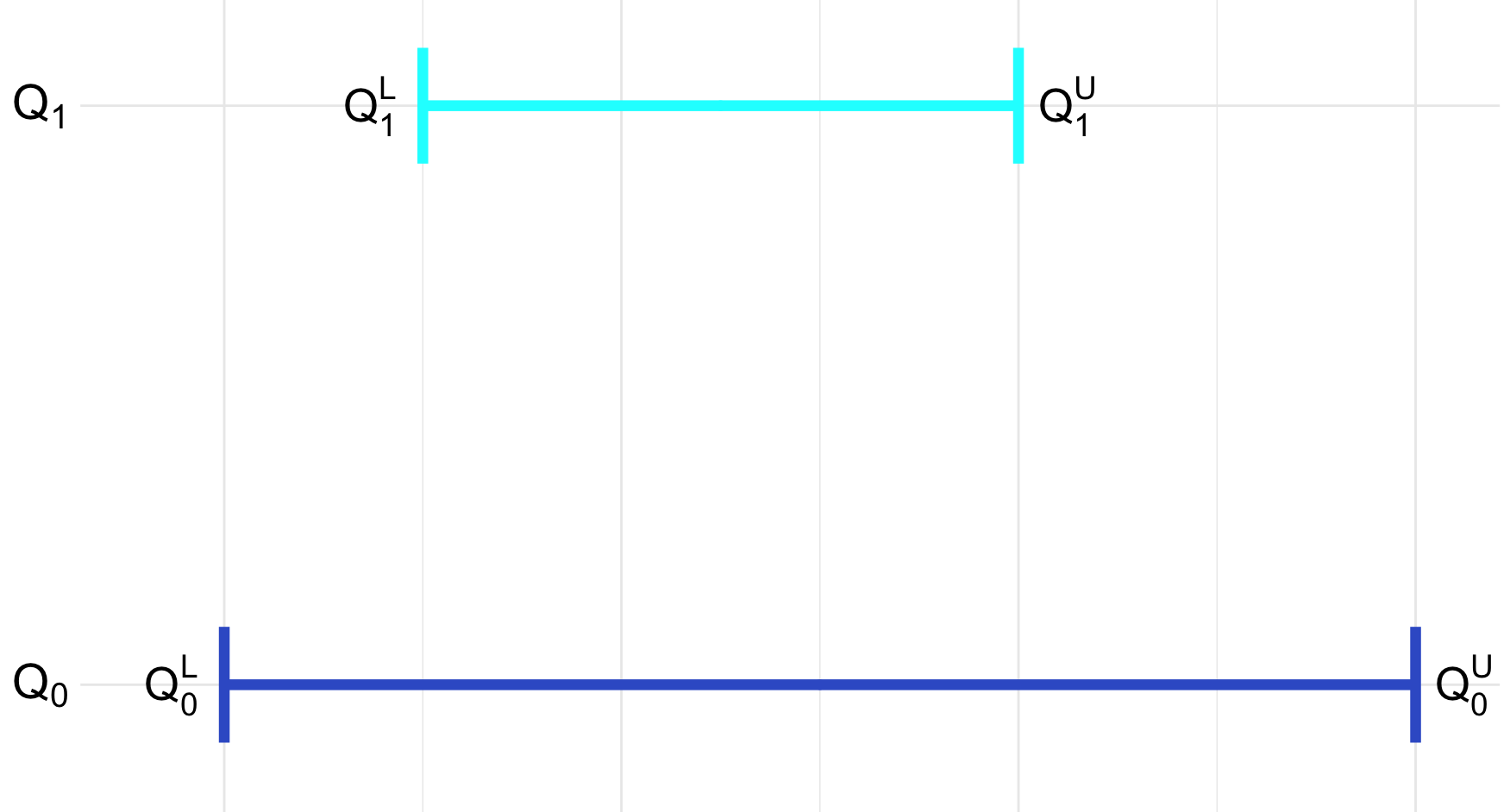}
    \caption{An illustrative example demonstrating the partial identification task, where only intervals of the revenue under the two prices can be obtained: $Q_a \in [Q_a^L,Q_a^U]$, for $a=0 \mbox{ and }1$. }
    \label{fig:pess vs oppor}
\end{figure}

\subsection{Partial Identification}

We begin by defining identifiability.
\begin{definition}[Identifiability]  \label{def:partial}
    A parameter $\lambda$ in a probabilistic model $\{\mathcal{P}_\lambda: \lambda \in \Lambda \}$ is said to be identifiable if the mapping $\lambda \mapsto \mathcal{P}_\lambda$ is injective. In other words, for any $\lambda_1 \neq \lambda_2$, we have $\mathcal{P}_{\lambda_1} \neq \mathcal{P}_{\lambda_2}$. 
\end{definition}

Identifiability ensures that if you observe data from the model, there is no ambiguity in determining which parameter $\lambda$ generated the data.  By Definition \ref{def:partial}, if a specific price $a$ is never observed at time $t$ (i.e., $\p^{\pi^b}_t(a) = 0$), then the conditional expectation $\E(D_t \mid A_t = a)$ is not well-defined and therefore cannot be identified. Consequently, the associated expected revenue is also unidentified. This lack of identifiability poses a fundamental challenge for learning the optimal policy. To address this, our new approach is based on the following monotonic assumption regarding the relationship between the price and the demand, which is natural for pricing problems. 

\ZB{\begin{assumption}[Monotonicity] \label{assumption:mono}  For any $t$, any $a < a^+$, and any $d$, the conditional CDF satisfies 
\[
    F_t(d \mid a) \;\leq\; F_t(d \mid a^+).
\]
\end{assumption}
Assumption \ref{assumption:mono} implies that a higher price corresponds to a higher CDF, which is plausible and testable in many business applications. For example, in retail markets, higher prices for consumer goods such as clothing and electronics often lead to reduced demand. We note that the monotonicity assumption imposed here is weaker than the commonly used  demand models in prior work, including linear demand model \citep{bu2023offline}, the binary purchasing model \citep[see, e.g.,][]{luo2024distribution,bian2026beyond}, as they directly imply monotonicity.

For an unobserved price $a$, the corresponding CDF of demand, $F_t(\cdot|a)$, is not identifiable by Definition \ref{def:partial}. Assumption \ref{assumption:mono} allows us to construct a partial-identification bound for $F_t(\cdot|a)$ by using other observed prices as proxies.
Specifically, at time $t$,
consider the scenario where the prices satisfy $a^+ > a > a^-$, with $\p^{\pi^b}_t(a) = 0$, and where $a^+$ and $a^-$ are in the set of observed prices $\mathcal{A}_t^{\pi^b}\equiv \{a:   \p^{\pi^b}_t(a) > 0\}$ at the population level. Under Assumption \ref{assumption:mono}, we have $ F_t(d|a^-) \leq F_t(d|a) \leq  F_t(d|a^+)$, which is a partial-identification bound because both the upper and lower bounds are identifiable. In addition, we have the following tightest (i.e., sharp) partial-identification bound for $F_t(d|a)$:
\begin{gather*}
   F_t\left( d|a_t^l \right) \leq F_t(d|a) \leq  F_t\left(d|a_t^u\right),  
   \mbox{ where } a_t^l=\max_{a' \in \mathcal{A}_t^{\mathcal{D}_N}: a'< a} a', \mbox{ and }a_t^u =\min_{a' \in \mathcal{A}_t^{\mathcal{D}_N}: a' > a} a' .
\end{gather*}  That is, at each time $t$, its tightest  lower and upper bounds for price $a$ are constructed utilizing the neighboring prices above and below $a$, in the set $\mathcal{A}_t^{\pi^b}$.  We next illustrate two methods of using the offline data to construct the estimated partial identified bounds.}

\begin{remark}
    So far, we have introduced three sets of prices: $\mathcal{A}$, $\mathcal{A}_t^{\mathcal{D}_N}$, and $\mathcal{A}_t^{\pi^b}$. It is important to distinguish them: $\mathcal{A}$ represents the set of all possible prices, regardless of their observability; $\mathcal{A}_t^{\mathcal{D}_N}$ denotes the set of observed prices at time $t$ in the offline data; and $\mathcal{A}_t^{\pi^b}$ indicates the set of observed prices at time $t$ under the offline data distribution following the behavior policy. Note that the two sets $\mathcal{A}_t^{\mathcal{D}_N}$ and $\mathcal{A}_t^{\pi^b}$ are asymptotically equivalent, but they may differ in the finite sample setting, due to the  sampling uncertainty. By applying the union bound, one can show that with probability at least $1-\kappa_t$, the set $\mathcal{A}_t^{\mathcal{D}_N}$ matches  $\mathcal{A}_t^{\pi^b}$, where $\kappa_t=\sum_{a \in \mathcal{A}_t^{\pi^b}}\left(1-\p_t^{\pi^b}(a)\right)^N$.
\end{remark}

\ZB{
	\textbf{\textit{A naive method}}. In practice, when $a$ is unobserved,
the sharp partial identification bound for $F_t(d|a)$ can be estimated using the closest available prices below and above $a$ in the set $\mathcal{A}_t^{\mathcal{D}_N}$: \begin{gather} 
     \notag \text{Estimated lower bound:} \ \, \widehat F_t\left(d|a^l_t \right)=\frac{1}{N_t(a_t^l)}\sum_{i=1}^N  \mathds 1\left(D_{i,t}\leq d, A_{i,t}=a_t^l \right),  \\
    \text{Estimated upper bound:}\ \, \widehat F_t\left(d|a^u_t\right)=\frac{1}{N_t(a_t^u)}\sum_{i=1}^N  \mathds 1\left(D_{i,t}\leq d, A_{i,t}=a_t^u\right).
 \label{eq:crude CI} 
\end{gather}  

Due to the sampling error, $\widehat F_t(d|a_t^l)$ is not necessarily less than or equal to $\widehat F_t(d|a_t^u)$. Furthermore, even if the interval is valid, there is no guarantee that $F_t(d|a)$ will fall within it. To resolve this, one can instead use the interval $\left[\widehat F_t^l(d|a_t^l),\widehat F_t^u(d|a_t^u)\right]$, 
where
\begin{gather*}
\widehat F_t^l(d|a_t^l)=\max\left[\widehat F_t(d|a_t^l)-\delta_t(a_t^l),\varepsilon \right], \widehat F_t^u(a_t^u)=\min\left[\widehat F_t(d|a_t^u)+\delta_t(a_t^u),1-\varepsilon\right],
\end{gather*} 
$\delta_t(a)=c\sqrt{\log N/N_t(a)}$, for some positive constant $c$, and $\varepsilon$ is a sufficiently small positive constant. This ensures that, with high probability, 
$F_t(d|a) \in \left[\widehat F_t^l(d|a_t^l),\widehat F_t^u(d|a_t^u)\right]$.
To see this, note that by Dvoretzky–Kiefer–Wolfowitz inequality, there exists some universal constant $c$ such that with probability at least $1-T/N$: \begin{gather*}
    \widehat F_t(d|a_t^l)-\delta_t(a_t^l)\leq  F_t(d|a_t^l) \leq F_t(d|a)\leq  F_t(d|a_t^u) \leq \widehat F_t(d|a_t^u)+\delta_t(a_t^u),
\end{gather*}  uniformly over $t$ and $d$.

However, the confidence bounds for $F_t(d|a)$ mentioned above are not the tightest possible. This is because there may exist prices $\widetilde{a} > a_t^u$ and $\bar{a} < a_t^l$ with smaller value of $\delta_t$ 
in estimating the corresponding CDFs, such that $\widehat F_t^u(d|\widetilde{a})\leq \widehat F_t^u(d|a_t^u)$ and $\widehat F_t^l(d|\bar{a})\geq \widehat F_t^l(d|a_t^l)$. Therefore, to account for the sampling errors, when $a$ is unobserved in the offline data, one should not just use the nearest prices below and above $a$ to construct the corresponding partial identification bound. Following the same reasoning, even if $a$ is observed from the data, $[\widehat  F_t^l(a),\widehat F_t^u(a)]$ itself might not necessarily be the shortest interval for $F_t(d|a)$. This motivates us to develop a refined approach given below. 

	\textbf{\textit{A refined interval}}. To deal with the issue mentioned above, we propose using the following interval $\left[\widehat F_t^L(d|a),\widehat F_t^U(d|a)\right]$ for $F_t(d|a)$, which is the tightest possible:  \begin{gather} \label{eq:CI}
\begin{split}
   \text{estimated lower bound:} \; \widehat F_t^L(d|a)= \begin{cases}
        \max_{a'< a}\widehat F_t^l(d|a'), & \mbox{ if } \left\{\widehat F_t^l(d|a')\right\}_{a'< a} \mbox{ is non-empty },\\
        \varepsilon,& \mbox{ otherwise}.
    \end{cases} \\
     \text{estimated upper bound:} \;  \widehat F_t^U(d|a)= \begin{cases}
        \min_{a'> a}\widehat F_t^u(d|a'), & \mbox{ if } \left\{\widehat F_t^u(d|a')\right\}_{a'> a} \mbox{ is non-empty },\\
        1-\varepsilon,& \mbox{ otherwise},
    \end{cases}
    \end{split}.
\end{gather}  In other words, the upper (lower) bound of $F_t(d|a)$ for price $a$ is calculated by taking the minimum (maximum) of the upper (lower) bounds for all demand CDFs whose prices are higher (lower) than $a$. Note that in the case that $\max_{a'< a}\widehat F_t^l(d|a')$ or $\min_{a'> a}\widehat F_t^u(d|a')$ is not well defined, i.e., there is no price from which we can borrow information to lower/upper bound $F_t(d|a)$, we use $\varepsilon$ and $1-\varepsilon$ to lower/upper bound $F_t(d|a)$. When $\left\{\widehat F_t^l(d|a')\right\}_{a'< a}$ or $\left\{\widehat F_t^u(d|a')\right\}_{a'> a}$ is non-empty, denote $\widehat a_t^U$ and $\widehat a_t^L$ as the prices used to construct the upper and lower bounds of $F_t(d|a)$, respectively, i.e., $\widehat F^U_t(d|a)= \widehat F_t(d|\widehat a_t^U) +\delta_t(\widehat a_t^U)$, and $\widehat F^L_t(d|a)= \widehat F_t(d|\widehat  a_t^L) -\delta_t(\widehat a_t^L)$. 

By similar arguments as before, we can show that for each $a \in \mathcal A$, $F_t(d|a)$ lies within $\left[\widehat F_t^L(d|a),\widehat F_t^U(d|a)\right]$ with probability at least $1-T/N$, uniformly over $t$ and $d$. Our remaining task is to derive the optimal policy under this partial identification framework. }

% \begin{figure}[t]
%     \centering
% \includegraphics[height=7cm,
% width=\textwidth]{mono.png}
%     \caption{Illustration of the construction of  the confidence interval for an unobserved price $a$ using the monotonic relationship between price and demand parameter.}
%     \label{fig:mono}
% \end{figure}

% see Figure \ref{fig:mono} for further illustration

%In addition, the confidence interval  $\left[\widehat F_t^l(a),\widehat F_t^u(a)\right]$ of the demand parameter $F_t(d|a)$ for all $a$ can be further refined. 

\section{Pessimistic and Opportunistic Strategies} \label{sec:algorithms}

In this section, we provide our solutions of finding $\pi^\ast$ under the newly proposed partial identification framework. Specifically, we introduce the pessimistic approach for conservatively estimating an optimal policy in Section \ref{sec:pess2} and the opportunistic approach for finding an optimal policy that minimizes the maximum regret in Section \ref{sec:mm}.

\subsection{A Refined Pessimistic Strategy} \label{sec:pess2}

\ZB{We first define the confidence set $\Omega_t(x,a)$ for the conditional CDF as 
\begin{equation} \label{eq:omega}
   \Omega_t(x,a)=  \Bigl\{
    F = (F^{(0)}, \ldots, F^{(x-1)}) :
    F^{(d)} \in  \left[\widehat F_t^L(d \mid a),  \widehat F_t^U(d \mid a)\right],~
    F^{(0)} \le \cdots \le F^{(x-1)},~ d \le x-1
    \Bigr\}.
\end{equation}
That is, for each coordinate $d$, the component $F^{(d)}$ lies within its corresponding
interval $[\widehat F_t^L(d \mid a),\, \widehat F_t^U(d \mid a)]$, while preserving
the monotonicity property of a valid CDF. 
Given this confidence set $\Omega_t(x,a)$ for the conditional CDF $F_t(\cdot \mid a)$,
we next introduce a refined pessimistic strategy for estimating the optimal policy. 

Specifically, for any stage $t$, inventory level $x$, price $a$, and 
$F = (F^{(0)}, F^{(1)}, \ldots, F^{(x-1)})$, 
let $\widehat Q_t(x,a;F)$ denote the estimated Q-function obtained by substituting 
$F_t(d \mid a)$ in the Bellman equation with a generic constant value $F^{(d)}$ 
for all $d \le x-1$. 
Formally,
\begin{align} \label{eq:Q-function at fixed F}
    \widehat Q_t(x,a;F)
    = a x 
    + \sum_{d=0}^{x-1}
    \bigl[
        \widehat V_{t+1}(x-d)
        - \widehat V_{t+1}(x-d-1)
        - a
    \bigr] F^{(d)},
\end{align}
for any estimated value function $\widehat V_{t+1}(\cdot)$.
When each $F^{(d)}$ equals the true CDF value $F_t(d \mid a)$, 
we recover $\widehat Q_t(x,a;F) = (\mathcal{B}_t \widehat V_{t+1})(x,a)$.
We also denote by $\mathcal{B}_t^F$ the modified Bellman operator defined as
$(\mathcal{B}_t^F \widehat V_{t+1})(x,a) = \widehat Q_t(x,a;F)$.

Recall that based on the pessimistic principle, at each time point, instead of directly maximizing the estimated Q-function, a lower bound of the estimated Q-function is maximized. Here, the way we construct the corresponding lower bound for our specific pricing problem is slightly different from the vanilla pessimistic approach. Specifically, at time point $t$, let $\widehat Q_t^{\mathbf{pess}}(x,a;F) = (\B_{t}^F \widehat V_{t+1}^{\mathbf{pess}}) (x,a)$ be the estimated Q-function following our pessimistic strategy after time $t$, with the demand CDF fixed at $F$. Here, the superscript $\mathbf{pess}$ indexes the policy, i.e., $Q_t^{\mathbf{pess}} \equiv Q_t^{\pi^{\mathbf{pess}}}$. We propose to use $\min_{F \in \Omega_t(x,a)} \widehat Q_t^{\mathbf{pess}}(x,a;F)$ as the lower bound for $Q^\ast_t(x, a)$. Note that the optimization problem $\min_{F \in \Omega_t(x,a)} \widehat Q_t^{\mathbf{pess}}(x,a;F)$ is a linear program, since both the constraint set \eqref{eq:omega} and the objective function \eqref{eq:Q-function at fixed F} are linear in $F$. Therefore, it can be efficiently solved using off-the-shelf optimization software.
Subsequently, the pessimistic action is given by \begin{gather*}
    \widehat a_{t}^{\mathbf{pess}}(x)=\argmax_{a\in \mathcal A}  \widehat Q_t^{\mathbf{pess}}(x,a), \mbox{ with } \widehat Q_t^{\mathbf{pess}}(x,a)=\min_{F \in \Omega_t(x,a)} \widehat Q_t^{\mathbf{pess}}(x,a;F).
\end{gather*}} Correspondingly, the pessimistic value function is estimated by $\widehat V_{t}^{\mathbf{pess}}(x)= \widehat Q_{t}^{\mathbf{pess}}(x,\widehat a_{t}^{\mathbf{pess}}(x))$. Then at time $t-1$, we have $\widehat Q_{t-1}^{\mathbf{pess}}(x,a)=\min_{F \in \Omega_t(x,a)}(\B_{t-1}^F \widehat V_{t}^{\mathbf{pess}}) (x,a)$. The above procedure is iterated from $t=T$ to $t=1$. In Proposition \ref{prop:pess} of Appendix, we demonstrate that $\widehat{Q}_t^{\mathbf{pess}}(x,a)$ serves as a valid lower bound for the optimal Q-function $Q_t^*(x,a)$ with a high probability, for all $t$, $x$, and $a$, confirming that $\widehat{Q}_t^{\mathbf{pess}}(x,a)$ is indeed a pessimistic estimator of the optimal Q-function.

We summarize our proposed pessimistic strategy in the partial identification framework in Algorithm \ref{alg:2}.

\begin{algorithm}[H]
	\caption{Pseudocodes for refined pessimistic approach 
  %$\widehat\pi^{\mathbf{pess}}$.
  }\label{alg:2}
	\begin{algorithmic}[1]
		\small
		\Function{}{$ \mathcal{D}_N$}
   \State Construct confidence intervals $\Omega_t(x,a)$ for all $x$, $a$ and all $t$ using Equation \eqref{eq:CI}.
            
        \State Set $\widehat V_{T+1}^{\mathbf{pess}} \gets 0$.

        \For {$t=T, T-1, \dots , 1$,}

        \State $\widehat Q_{t}^{\mathbf{pess}}(x,a;F)=\begin{cases}
			0, & \text{if $x=0$},\\
   (\B_{t}^F \widehat V_{t+1}^{\mathbf{pess}}) (x,a),& \text{otherwise.}
		 \end{cases}$.

           \State  Compute $ \widehat Q_{t}^{\mathbf{pess}}(x,a)=
       \min_{F \in \Omega_{t}(x,a)} \widehat Q_{t}^{\mathbf{pess}}(x,a;F).$
        
        \State Compute $\widehat a_{t}^{\mathbf{pess}}(x)=\argmax_{a \in \mathcal{A}} \widehat Q_{t}^{\mathbf{pess}}(x,a)$, and $\widehat V_{t}^{\mathbf{pess}}(x)= \widehat Q_{t}^{\mathbf{pess}}(x,\widehat a_{t}^{\mathbf{pess}}(x))$.

        \State Compute $\widehat \pi_t^{\mathbf{pess}}(a|x) = \begin{cases} 
1, & \text{if } a = \widehat a_{t}^{\mathbf{pess}}(x), \\
0, & \text{otherwise}. \end{cases}$
\EndFor

\State \Return $\widehat\pi^{\mathbf{pess}}=\{\widehat \pi^{\mathbf{pess}}_t \}_{1\leq t\leq T}$.

		\EndFunction
\end{algorithmic}

\end{algorithm}

We now present the theoretical properties of our proposed pessimistic method. The following Lemma \ref{lemma:decomp} plays an important role in establishing our regret bound. Refer to Section \ref{appendix:decomp} in the Appendix for a detailed proof.

\begin{lemma}[Decomposition Lemma] \label{lemma:decomp} For all $t$ and any estimated policy $\{\widehat \pi_t\}_{1\leq t\leq T}$, as well as the estimated Q-function $\widehat Q_t(x,a)$, the following holds:
    \begin{align*}
        \mu^{\widehat \pi} = \underbrace{\sum_{t=1}^T \E^{\pi^*}[ l_t(X_t,A_t) ]}_{J_1}- \underbrace{\sum_{t=1}^T \E^{\widehat \pi}[ l_t(X_t,A_t) ]}_{J_2} +\underbrace{\sum_{t=1}^T \E^{\pi^*}\left[ \sum_{a\in \mathcal A} \widehat Q_t(X_t,a)\left(\pi^*_t(a|X_t)-\widehat\pi_t(a|X_t)\right) \right]}_{J_3},
    \end{align*} where $l_t(x,a)=( \B_t \widehat V_{t+1})(x,a)- \widehat Q_{t}(x,a)$ and $\E^\pi$ means that the expectation is taken by assuming the system dynamics follows the policy $\pi$, with the functions $\widehat V_t(\cdot)$ and $\widehat Q_t(\cdot,\cdot)$ being held fixed.
\end{lemma}

\begin{remark}
    Lemma \ref{lemma:decomp} is algorithm-agnostic and can be applied to any  estimated value functions $\widehat V_t$, estimated Q-functions $\widehat Q_t$, and estimated policy $\widehat \pi$. To apply Lemma \ref{lemma:decomp} for the proposed pessimistic approach, one only needs to replace them by $\widehat V_t^{\mathbf{pess}}$, $\widehat Q_t^{\mathbf{pess}}$ and $\widehat \pi^{\mathbf{pess}}$, respectively.
\end{remark}

 Lemma \ref{lemma:decomp} implies that the regret 
 can be decomposed into
three terms. In particular, the first and second terms depend on the optimal policy and the estimated policy, respectively. As shown in the proof in the Appendix, by employing the pessimistic strategy, $l_t(x,a)$ is guaranteed to be non-negative with high probability. Combining this with the fact that $\widehat \pi^{\mathbf{pess}}$ is greedy with respect to $\widehat Q_t^{\mathbf{pess}}(x,a)$, we have $J_2 \geq 0$, $J_3 \leq 0$ and thus $\mu^{\mathbf{pess}} \leq J_1$. It thus remains to upper-bound the first term $J_1$. The final regret bound is given in the following theorem. The proof is deferred to the Appendix.

\ZB{
\begin{theorem} \label{thm: regret 2}
       With probability at least $1-\sum_{t=1}^T \left( \frac{|\mathcal{A}_t^{\mathcal{D}_N}|}{N}+\kappa_t \right)$, the regret function for our proposed refined pessimistic method satisfies that 
        \begin{align*}
            \mu^{\mathbf{pess}} \lesssim &  \sum_{t=1}^T \E^{\pi^b}\left[ \frac{\prob_t^{\pi^*}(X_t,A_t)}{\prob_t^{\pi^b}(X_t,A_t)} \sqrt{\log N/N_t(A_t) } \mathds{1}(A_t \in \mathcal{M}_t(X_t))  \right]\\
            &+ \sum_{t=1}^T\E^{\pi^*}\left[ \eta_t (X_t,A_t)\mathds{1}(A_t \notin \mathcal{M}_t(X_t)) \right],
        \end{align*}
     where  
    $\mathcal{M}_t(x) = \{a\in \mathcal{A}: \prob_t^{\pi^*}(x,a) >0 \implies \prob_t^{\pi^b}(x,a)>0 \} $,  and
    \begin{gather*}
    \eta_t (x,a) = \begin{cases}
			 \sum_{d=0}^{x-1}\left[F_t(d|a_t^u) - F_t(d|a_t^l)\right]+2(x-1)\left[\delta_t(a_t^u)+\delta_t(a_t^l)\right], & \text{ if }  a_t^U \mbox{ and } a_t^L \mbox{ exist},  \\
             (x-1)\left[2\delta_t(a_t^l)+1\right] - \sum_{d=0}^{x-1} F_t(d|a_t^l), & \text{ if only }  a_t^L \mbox{ exists},\\
             \sum_{d=0}^{x-1} F_t(d|a_t^u) +2(x-1)\delta_t(a_t^u), & \text{ if only }  a_t^U \mbox{ exists },  
		 \end{cases} \\
   a_t^U=\begin{cases}
   a, & \text{ if } a \in \mathcal{A}_t^{\mathcal{D}_N}\\
   a_t^u, & \text{ otherwise },
   \end{cases}, \mbox{ and }  \;
   a_t^L=\begin{cases}
   a, & \text{ if } a \in \mathcal{A}_t^{\mathcal{D}_N}\\
   a_t^l, & \text{ otherwise },
   \end{cases}.
\end{gather*}

In addition, the above upper bound is tight in the sense that it is attainable, i.e., there exists a true data-generating process compatible with the observed data distribution, whose regret matches the derived regret bound.
\end{theorem} }

Theorem \ref{thm: regret 2} establishes that the regret bound consists of
two elements: the first element arises from the estimation of the demand CDFs for observed prices in the data; while the second term accounts for the error due to the unobserved optimal price. Note that the probability terms arise from two components: the first, $1/N$ comes from applying a concentration inequality to upper bound the estimation error of the partial identification bound; the second, $\kappa_t$ bounds the probability that the set $\mathcal{A}_t^{\mathcal{D}_N}$ differs from $\mathcal{A}_t^{\pi^b}$, thereby allowing us to establish the tight regret bound. When the behavior policy covers the optimal policy (i.e., the partial coverage assumption defined in Equation \eqref{eq:partial}), the second element in the regret is zero and the ratio $\prob_t^{\pi^*}(X_t, A_t)/\prob_t^{\pi^b}(X_t,A_t) $ is bounded above, resulting in a regret of order $O_p\left(\sqrt{\log N/N}\right)$. Informally speaking, our regret bound is optimal because the first term matches the optimal rate (Corollary 4.5 in \citealt{jin2021pessimism}), under the partial coverage assumption. The second term is also rate-optimal, as the interval $\left[ F_t(d|a_t^l), F_t(d|a_t^u)\right]$ is the sharpest partial identification interval for $F_t(d|a)$, conditional on the event $\{\mathcal{A}_t^{\mathcal{D}_N}=\mathcal{A}_t^{\pi^b}\}$, making it attainable as well. For a detailed explanation, refer to the Appendix.

%\done{\LW{Maybe we can say under the partial coverage condition in Section 2.3, the second term in the regret bound is zero. If we want to claim the rate 
%$O_p\left(\sqrt{\log N/N}\right)$, we may need to discuss more the condition in Jin et al. If we want to discuss the regret bound rate as in several places in the paper, we'll need to discuss the rate of the second term in the regret bound, and be mindful that the rate $O_p\left(\sqrt{\log N/N}\right)$ may not be achievable in the general no coverage setting. It seems to me Jin et al. showed PEVI is minimax optimal, but does not claim the $O_p\left(\sqrt{\log N/N}\right)$ rate in general.}}

% \ZB{The condition in corollary in Jin is equivalent to: there exists some constant $c$ such that
% \begin{gather*}
%   \sup_{\norm{b}{2}=1}  b^\top \left[\E^{\pi^*}(\phi \phi^\top)-c\E^{\pi^b} (\phi \phi^\top) \right] b <0.
% \end{gather*} The above holds if $\prob_t^{\pi^*}(x,a)/\prob_t^{\pi^b}(x,a) $ is bounded. To summarize, given $\prob_t^{\pi^*}(x,a)/\prob_t^{\pi^b}(x,a) $ is bounded, our regret is $O_p(\sqrt{\log N /N})$, which matches the bound in Jin, and is also minimax optimal.
% } 

\subsection{Opportunistic Strategy} \label{sec:mm}

As we have discussed earlier in Section \ref{sec:intro}, the pessimistic principle is argued to be ultra-pessimistic \citep{savage1951theory}, where the decision maker deems the world to be in the worst possible state. Indeed, if a specific price yields the revenue with a small lower bound, a pessimist would not consider it regardless of its potential upper bound. Consequently, a purely pessimistic strategy may be overly conservative, 
potentially causing missed opportunities. 
In contrast, the opportunist's strategy, inspired by the pessimistic approach,
advocates choosing the action that minimizes the maximum regret. By doing this, a decision maker might still consider the price with a small lower bound if the upper bound 
is sufficiently high, thus seizing potential
opportunities.

To build intuition, revisit the two-armed setting in Figure~\ref{fig:pess vs oppor}. Let $Q_0$ and $Q_1$ denote expected revenues under actions $0$ and $1$, with bounds $Q_a \in [Q_a^L, Q_a^U]$. A pessimist compares $Q_0^L$ and $Q_1^L$, while an opportunist compares each action by its worst-case regret relative to the best alternative.

\begin{example} \label{example:pess vs opp}
    In the setting of Figure~\ref{fig:pess vs oppor}, pessimists select arm $1$ because $Q_1^L > Q_0^L$. Opportunists select arm $0$ because the maximum regret of arm $1$ is $Q_0^U - Q_1^L$, while the maximum regret of arm $0$ is $Q_1^U - Q_0^L$, and $Q_1^U - Q_0^L < Q_0^U - Q_1^L$.
\end{example}

While both strategies address the worst-case scenarios, the key difference lies in their focus: pessimists prioritize minimizing the worst-case revenue, while opportunists focus on minimizing the worst-case regret. Opportunists view a purely pessimistic approach as overly conservative and potentially limiting opportunities. Consider the action $a_1$, chosen based on the pessimistic principle. An opportunist wants to know: what if I choose 
$a_1$ and the environment turns out to be less unfavorable? What if 
$a_1$ performs significantly worse than other actions in such a scenario? An opportunist aims to avoid the scenario that in the worst environment, both 
$a_1$ and $a_0$ perform poorly, with 
$a_1$ being marginally better, while in other better environments, $a_1$ performs much worse than $a_0$. An opportunist would choose 
$a_0$ instead, as it behaves only slightly worse than  $a_1$ in one environment but significantly better in another, making it an opportunistic choice.

Example~\ref{example:pess vs opp} illustrates why minimizing worst-case regret can differ from maximizing a lower bound, and this distinction extends to the sequential setting below. For the sequential decision-making setting, at the $t$-th iteration of the backward induction, the estimated opportunistic price is $\widehat a_t ^{\mathbf{opp}}$ which is obtained by computing \ZB{\begin{align} \label{eq: minimax}
    \argmin_{a \in \mathcal A} \max_{F' \in \Omega_t(x, a'), F \in \Omega_t(x,a)}\left [\max_{a' \neq a}\widehat Q_t^{\mathbf{opp}}(x,a';F')-\widehat Q_t^{\mathbf{opp}}(x,a;F)\right],
\end{align} where $\widehat Q_t^{\mathbf{opp}}(x,a;F) \equiv (\B_{t}^F \widehat V_{t+1}^{\mathbf{opp}}) (x,a)$ is the estimated Q-function following opportunistic strategy after time $t$,
with the demand CDF fixed at $F$; and $ \widehat V_{t}^{\mathbf{opp}}(x)$ is the opportunistic value function. The term \begin{gather*}
    \max_{F' \in \Omega_t(x, a'), F \in \Omega_t(x,a)}\left [\max_{a' \neq a}\widehat Q_t^{\mathbf{opp}}(x,a';F')-\widehat Q_t^{\mathbf{opp}}(x,a;F)\right]
\end{gather*} is the estimated maximum regret of action $a$ at time $t$ with inventory level $x$, where the outer maximization over $F$ and $F'$ means that we consider the worst case with respect to the \textbf{\textit{regret}}. Here, the superscript $\mathbf{opp}$ also indexes the policy i.e., $Q_t^{\mathbf{opp}} \equiv Q_t^{\pi^{\mathbf{opp}}}$. Finally, by taking the minimum over the action space $\mathcal A$, this yields the opportunistic decision. Denote the maximizer of $F$ over $\Omega_t(x,a)$ for price $a$ as $ F_t^{\mathbf{opp}}(a),$ and the opportunistic action as $\widehat a_t^{\mathbf{opp}}$. Subsequently, the opportunistic value function is $ \widehat V_{t}^{\mathbf{opp}}(x)= \widehat Q_{t}^{\mathbf{opp}}(x,\widehat a_{t}^{\mathbf{opp}}(x);F_t^{\mathbf{opp}}(\widehat a_{t}^{\mathbf{opp}}))$. The same procedure repeats in a backward manner: $\widehat Q_{t}^{\mathbf{opp}}(x,a;F)=(\B_{t}^F \widehat V_{t+1}^{\mathbf{opp}}) (x,a)$, until $t = 1$. The entire estimation procedure is summarized in Algorithm \ref{alg:3}.}

\begin{algorithm}[H]
	\caption{Pseudocodes for Estimating 
  $\widehat\pi^{\mathbf{opp}}$.}\label{alg:3}
	\begin{algorithmic}[1]
		\small

           \Function{}{$ \mathcal{D}_N$}
           
   \State Construct the confidence interval $\Omega_t(x,a)$ for all $a$ and all $t$ using Equation \eqref{eq:CI}.
   
        \State Set $ \widehat V_{T+1}^{\mathbf{opp}} \gets 0$.

        \For {$t=T, T-1, \dots , 1$}

        \State $\widehat Q_{t}^{\mathbf{opp}}(x,a;F)=\begin{cases}
			0, & \text{if $x=0$}\\
   (\B_{t}^F \widehat V_{t+1}^{\mathbf{opp}}) (x,a) & \text{otherwise}
		 \end{cases}.$

        \State  
            $\widehat \mu_{t}^{\mathbf{opp}}(x,a)=
   \max_{F' \in \Omega_{t}(a'), F \in \Omega_{t}(a)}\left [\max_{a' \neq a}\widehat Q_t^{\mathbf{opp}}(x,a';F')-\widehat Q_t^{\mathbf{opp}}(x,a;F)\right].$ 
        
        \State Compute $ \widehat a_{t}^{\mathbf{opp}}(x)=\argmin_{a \in \mathcal{A}} \widehat\mu_{t}^{\mathbf{opp}}(x,a) $, and $ \widehat V_{t}^{\mathbf{opp}}(x)= \widehat Q_{t}^{\mathbf{opp}}(x,\widehat a_{t}^{\mathbf{opp}}(x);F_t^{\mathbf{opp}}(\widehat a_{t}^{\mathbf{opp}}))$.

        \State Compute $\widehat \pi_t^{\mathbf{opp}}(a|x) = \begin{cases} 
1 & \text{if } a = \widehat a_{t}^{\mathbf{opp}}(x), \\
0 & \text{otherwise}, \end{cases}$

\EndFor

\State \Return $\widehat\pi^{\mathbf{opp}}=\{\widehat \pi^{\mathbf{opp}}_t \}_{1\leq t\leq T}$

		\EndFunction
\end{algorithmic}

\end{algorithm}

\begin{remark}
It is worth noting that the optimal policy under the opportunistic criterion in the static setting ($T=1$) is shown to be stochastic, see, e.g., \citep{manski2007minimax,Cui2021Individualized}. Nevertheless, as we commented earlier, in pricing problems, it is natural to focus on deterministic policies.
\end{remark}

We next present the regret of our opportunistic approach in Theorem \ref{theorem:regret3}. The proof follows a similar approach to that of Theorem \ref{thm: regret 2}. We first apply Lemma \ref{lemma:decomp} for the opportunistic approach, where one only needs to plug $\widehat V_t^{\mathbf{opp}}(x)$, $\widehat \pi^{\mathbf{opp}}$ and $\widehat Q_t^{\mathbf{opp}}(x,a)$ into the equation in Lemma \ref{lemma:decomp}. Subsequently, we demonstrate that $l_t(x,a)$ is a non-negative term; next, we establish that the term $J_3$ in Lemma \ref{lemma:decomp} is non-positive. Finally, we provide an upper bound for the term $J_1$ in Lemma \ref{lemma:decomp}.

\begin{theorem} \label{theorem:regret3}
    With probability at least $1-\sum_{t=1}^T \left( \frac{|\mathcal{A}_t^{\mathcal{D}_N}|}{N}+\kappa_t \right)$, the regret function for the proposed opportunistic method satisfies that \begin{align*}
        \mu^{\mathbf{opp}} \lesssim  &  \sum_{t=1}^T \E^{\pi^b}\left[ \frac{\prob_t^{\pi^*}(X_t, A_t)}{\prob_t^{\pi^b}(X_t,A_t)} \sqrt{\log N/N_t(A_t) } \mathds{1}(A_t \in \mathcal{M}_t(X_t))  \right]\\
            &+  \sum_{t=1}^T\E^{\pi^*}\left( \eta_t (X_t,A_t)\mathds{1}(A_t \notin \mathcal{M}_t(X_t)) \right).
    \end{align*}
    Moreover,  the  upper bound is tight in the sense that it is attainable, i.e., there exists a true data-generating process compatible with the observed data distribution, whose regret matches the derived regret bound.
\end{theorem}

By Theorem \ref{theorem:regret3}, the regret of the opportunistic approach aligns with that of the refined pessimistic method.  Hence, when the behavior policy covers the optimal policy, it also yields a regret of order $O_p\left(\sqrt{\log N/N}\right)$. This upper bound is also attainable. The proof of Theorem \ref{theorem:regret3} is given in the Appendix.

\subsection{Comparison of Vanilla Pessimistic, Refined Pessimistic, and Opportunistic Methods}

In this subsection, we compare three methods: the vanilla pessimistic, refined pessimistic, and opportunistic approaches. We start with the two pessimistic methods and analyze them under two scenarios: (a) all prices are observed, and (b) some prices are unobserved.

\textbf{\textit{Full coverage setting}}. In this scenario, the refined interval constructed using Equation \eqref{eq:CI} provides a narrower range compared to the crude interval constructed by \eqref{eq:crude CI}, resulting in the refined pessimistic approach having a more accurate estimator of the Q-function than the vanilla pessimistic Q-function. However, in the static setting, the two pessimistic approaches yield the same estimated policy, see the following proposition.

\begin{prop} \label{prop:static pess}
   Under the static setting in which $T=1$, if all the prices are observed, then
   the vanilla and refined pessimistic approaches yield the same estimated policy and value.
\end{prop}

Intuitively, in the static setting, there is no need to balance the immediate revenue with future revenue. For any fixed price $a$, the expected revenue decreases as the demand CDF, $F(d|a)$, increases for all $d$. Thus, for each price, the worst-case scenario arises when $F(d|a)$ takes its highest value within the interval. Consequently, improvements in the crude lower bounds for certain prices do not impact the estimated policies for both pessimistic methods. Now consider the case where the upper bounds are refined, such that price $a_1$ borrows information from a higher price $a_2$, resulting in both prices sharing the same upper bound. In this scenario, the vanilla pessimistic method will choose $a_2$ because it offers a higher price and a lower crude upper bound; the refined pessimistic approach will also select $a_2$, since $a_1$ and $a_2$ share the same refined upper bound, and $a_1 < a_2$. By a similar argument, the two approaches yield the same value function.

However, this coincidence occurs only in the static setting. In the \textbf{\textit{dynamic}} setting, the two strategies can produce different solutions. This is because, for any fixed price, the expected revenue is not always a decreasing function of the demand CDF (see Proposition \ref{prop: bellman}). We illustrate this using the following example.

\begin{example}
    Suppose the initial inventory level is $1$ and $T=2$. By Proposition \ref{prop: bellman}, the estimated Q-function with $F^{(0)}$ fixed at time $1$ reduces to \begin{align*}
      \widehat Q_1(1,a;F^{(0)})=a +\left[  \widehat V_2(1)-a \right]F^{(0)}.
    \end{align*}  
By Proposition \ref{prop:static pess}, we have $\widehat V_2^{\mathbf{vp}}(1) = \widehat V_2^{\mathbf{pess}}(1)$. Therefore, we 
omit the superscripts and use $\widehat V_2(1)$ and $\widehat Q_1(1,a;F^{(0)})$ for both pessimistic methods. Assume that for any price $a \in \mathcal{A}$,  $\widehat V_2(1)>a$, then $\widehat Q_1(1,a;F^{(0)})$ is an increasing function of $F^{(0)}$. In such cases, the worst-case scenario occurs when $F^{(0)}$ takes its lowest value in the interval. Suppose we have two prices, $a_1<a_2$, and $a_2$ is the true optimal pricing at time $1$; $a_1$ is the price selected by the vanilla pessimistic approach, i.e., $ \widehat Q_1(1,a_1;\widehat F^l_1(0|a_1)) > \widehat Q_1(1,a_2;\widehat F^l_1(0|a_2))$. Consider the scenario that $\widehat F^l_1(0|a_1)>\widehat F^l_1(0|a_2)$, so that the refined lower bound for price $a_2$ is $\widehat F^l_1(0|a_1)$, i.e., borrowing the lower bound from $a_1$. Given that (i) $a_1$ and $a_2$ now share the same lower bound; (ii) the worst-case scenario occurs at this lower bound; and (iii) $a_1<a_2$, the refined pessimistic approach would correctly choose the optimal $a_2$.
\end{example}

We have shown how the two pessimistic approaches differ under the full coverage setting. Next, we differentiate them in a scenario where some prices are unobserved.

\textit{\textbf{Non-full coverage setting}}. If $a \notin \mathcal{A}_t^{\mathcal{D}_N}$, a confidence interval can still be constructed using Equation \eqref{eq:CI} for price $a$, allowing it to be considered in decision-making. In contrast, the vanilla pessimistic approach excludes price $a$ from the optimal pricing strategy. However, the following proposition reveals an interesting phenomenon: in many scenarios, the refined pessimistic approach is unlikely to choose the unobserved price.

\begin{prop} \label{prop:pess unable}
    Suppose $a_1 \notin \mathcal{A}_t^{\mathcal{D}_N}$. At time $t$, if there exists price $a_2 \in \mathcal{A}_t^{\mathcal{D}_N}$ such that the upper bound of the interval for $a_1$ is constructed using $a_2$ as a proxy, then the refined pessimistic approach would prefer $a_2$ over $a_1$, regardless of whether $a_1$ is optimal.
\end{prop}

We provide a brief explanation for Proposition \ref{prop:pess unable}.
If $a_2$ serves as a proxy for $a_1$ in constructing its upper bound, then it must hold that $a_1 < a_2$. By Equation \eqref{eq:CI}, we have that for all $x$, $\Omega_t(x,a_2) \subseteq \Omega_t(x,a_1)$. Thus, for any $F \in \Omega_t(x,a_1)$, $\widehat Q_t(x,a_2;F) \geq \widehat Q_t(x,a_1; F)$. Subsequently, we have  \begin{gather*}
    \min_{F \in \Omega_t(x,a_2)}\widehat Q_t(x,a_2, F)\geq \min_{F \in \Omega_t(x,a_2)}\widehat Q_t(x,a_1, F) \geq \min_{F \in \Omega_t(x,a_1)}\widehat Q_t(x,a_1, F),
\end{gather*} where the last inequality holds because of $ \Omega_t(x,a_2) \subseteq \Omega_t(x,a_1)$, for all $x$. We can thus conclude that the pessimistic strategy can never choose $a_1$ under this specific setting.

Although Proposition \ref{prop:pess unable} suggests that under the specific setting, an optimal price $a$ cannot be selected using a pessimistic approach; this finding does not imply that the refined pessimistic approach cannot improve the vanilla pessimistic approach in the presence of unobserved price. Example \ref{example: pess able} below demonstrates this point.

\begin{example}\label{example: pess able}
    For illustrative purposes, we consider a static setting in which $T=1$. By Proposition \ref{prop: bellman}, the worst-case scenario occurs when $F(\cdot|a)$ takes the value of the upper bound  within the interval, for all $a$. Consider the scenario where the unobserved optimal price  $a^*$ is the highest price in the set $\mathcal{A}$, meaning that no other price provides information to establish a upper bound on $F(\cdot|a^*)$. Given that $a^* \varepsilon   \geq \max_{a \in \mathcal{A}:a \neq a^*} a\widehat F^L(d|a) $, then the refined pessimistic approach would correctly pick $a^*$. 
\end{example}

In summary, our proposed refined pessimistic estimator offers a more accurate approximation of the Q-function for observed prices, reducing decision-making uncertainty. Furthermore, when the unobserved price is the highest, our refined pessimistic approach retains the potential to identify it, unlike the vanilla pessimistic method, which fails in this scenario. This is particularly advantageous, as in most applications, the highest prices are likely unobserved. Lastly, as demonstrated in Theorem \ref{thm: regret 2}, the regret of our refined pessimistic approach is rate optimal in the presence of unobserved price, further highlighting its superiority.

We next examine the opportunistic approach in the setting where some prices are unobserved. The opportunistic method can select any unobserved price if it has the potential for large expected revenue. We illustrate this point using the following example.

\begin{example} \label{example:opp}
    Consider a static setting with  prices $a_1 < a_2 $. If $a_1$ is unobserved, its demand CDF is partially identified within $\left[\varepsilon,\widehat F^U(d|a_2) \right]$. The opportunistic approach will select $a_1$ if \begin{gather} \label{eq:example opp}
        a_2\left[2x-\sum_{d=0}^{x-1} \left(\widehat F^U(d|a_2)+\widehat F^L(d|a_2)\right) \right] \leq a_1\left[2x-\sum_{d=0}^{x-1} \widehat F^U(d|a_2)-x\varepsilon \right].
    \end{gather} 
Similarly, when $a_2$ is unobserved, 
its demand CDF is partially identified within $\left[\widehat F^L(d|a_1), 1- \varepsilon \right],$ the opportunistic approach will select $a_2$ if \begin{gather*}
        a_1\left[2x-\sum_{d=0}^{x-1} \left(\widehat F^U(d|a_1)+\widehat F^L(d|a_1)\right) \right] \leq a_2\left[x-\sum_{d=0}^{x-1} \widehat F^L(d|a_1)+x\varepsilon \right].
    \end{gather*} 
Thus, when either $a_1$ or $a_2$ represents the true optimal price but is unobserved, the refined pessimistic approach will not select them (Proposition \ref{prop:pess unable}). In such cases, the opportunistic approach gives the decision maker a chance to seize the opportunity and achieve higher revenue.
\end{example} Example \ref{example:opp} demonstrates that, unlike the pessimistic approach, the opportunistic approach allows any price to be selected, thereby avoiding overly conservative actions. This advantage is further demonstrated in our numerical results in Section \ref{sec:sims}. While the opportunistic approach can be beneficial when the optimal price is unobserved, this advantage does not make it universally superior. In certain situations, being too aggressive may backfire, making the more cautious pessimistic approach, which offers downside protection, a better choice. 
\begin{example}
    Consider the same setting as in Example \ref{example:opp} with $a_2$ being the true optimal price and $a_1$ unobserved. By Proposition \ref{prop:pess unable}, the refined pessimistic approach will correctly select $a_2$. However, when $\varepsilon$ is small and can only provide a crude bound such that Equation \eqref{eq:example opp} holds, the opportunistic approach would incorrectly select $a_1$. 
\end{example}
The example above illustrates that being overly aggressive can be detrimental. To conclude this section, we note that the effectiveness of either strategy depends on the specific context and the decision-maker's risk tolerance. This is further discussed in Section \ref{sec:discussion}.

\section{Numerical Studies} \label{sec:sims}

\subsection{Synthetic Data} \label{sec:synthetic data}

In this section, we conduct a comprehensive simulation study to evaluate the empirical performance of the proposed methods. Throughout the experiments, we consider a setting with $N=50$ independent trajectories, each initialized with inventory level $X_{i,1}=15$. The action space consists of a discrete price grid $\mathcal{A} = \{1,2,\dots,10\}$, and we evaluate performance over time horizons $T \in \{10, 15, 20\}$. Throughout our implementation, to construct the partial identification bounds, we fix $\varepsilon$ in Equation \ref{eq:CI} at $0.02$ and set $\delta_t(a) = 0.1 \sqrt{\frac{\log N}{N_t(a)}}$. The state transition follows
$X_{i,t+1} = \max\{X_{i,t} - \min(D_{i,t}, X_{i,t}),\, 0\},$
and the reward is given by
$R_{i,t} = \min(D_{i,t}, X_{i,t}) \times A_{i,t}.$

\ZB{
\textbf{\textit{Demand models}}.
We investigate two demand distributions  in our simulations. First, we consider a Poisson model:
\[
D_{i,t} \mid A_{i,t} = a \sim \mathrm{Pois}(\lambda_t(a)),
\qquad \lambda_t(a) = \frac{11 - a}{2}, \quad \forall t.
\]
This specification ensures that the demand distribution (and hence its CDF) is monotone in the price. Second, we consider a Negative Binomial model:
\[
D_{i,t} \mid A_{i,t} = a \sim \mathrm{NB}(\eta_t(a), 10),
\]
where $\eta_t(a)$ denotes the mean demand and $10$ is the dispersion (size) parameter. We specify the mean function as
\[
\eta_t(a) = \frac{14 - 0.6a - 0.05a^2}{4}, \quad \forall t.
\]
This model allows for over-dispersion relative to the Poisson case while preserving a monotone relationship between price and the demand distribution.

\textbf{\textit{Behavior policies}}.
To thoroughly evaluate the robustness of the proposed estimator under a wide range of data-coverage patterns, we examine four scenarios with distinct behavior policies to represent different missing-data mechanisms.

Specifically, the first three policies are constructed by excluding different subsets of prices from the observed data; for all remaining prices in $\mathcal{A}$, the behavior policy assigns uniform sampling probabilities. In addition, we consider a bad policy designed to be suboptimal.

\noindent\textbf{Scenario 1.} Prices $\{1,5,10\}$ are missing.\\[0.5em]
\noindent\textbf{Scenario 2.} Prices $\{2,4,8\}$ are missing.\\[0.5em]
\noindent\textbf{Scenario 3.} Prices $\{3,6,7,9\}$ are missing.\\[0.5em]
\noindent\textbf{Scenario 4.} We consider a suboptimal policy that selects the rounded value of half the optimal price.

This design ensures that, across the four scenarios, every price is excluded at least once, thereby systematically covering a broad spectrum of missingness patterns. 

\textbf{\textit{Competing methods}}.
We compare the proposed method with the oracle optimal policy, the standard pessimistic approach, Conservative Q-Learning (CQL; \citealp{kumar2020conservative}), and Batch-Constrained Q-learning (BCQ; \citealp{fujimoto2019off}). For CQL, we adopt the formulation in Equation (4) of \citet{kumar2020conservative}, using a uniform distribution as the reference distribution. The original BCQ framework in \citet{fujimoto2019off} was designed for continuous action spaces and relies on parametrized models for both the policy and the Q-function. To better adapt the algorithm to our discrete pricing setting, we modify BCQ as follows. At each time step $t$, we generate a candidate action set $\mathcal{A}_t^{\mathrm{BCQ}}$ by sampling $|\mathcal{A}_t^{\mathcal{D}_N}|/2$ actions from the estimated behavior policy $\widehat\pi_t^b(\cdot \mid x)$. The action is then selected by maximizing the estimated Q-function over this restricted set:
\[
\widehat a_t^{\mathrm{BCQ}}(x) = \arg\max_{a \in \mathcal{A}_t^{\mathrm{BCQ}}} \widehat Q_t^{\mathrm{BCQ}}(x,a),
\]
where
\[
\widehat Q_t^{\mathrm{BCQ}}(x,a) = (\widehat{\mathcal{B}}_t \widehat V_{t+1}^{\mathrm{BCQ}})(x,a),
\qquad
\widehat V_t^{\mathrm{BCQ}}(x) = \widehat Q_t^{\mathrm{BCQ}}(x, \widehat a_t^{\mathrm{BCQ}}(x)).
\]
This ensures that the learned policy remains close to the support of the behavior policy.}

\begin{figure}[t]
    \centering
    \includegraphics[width=0.9\textwidth]{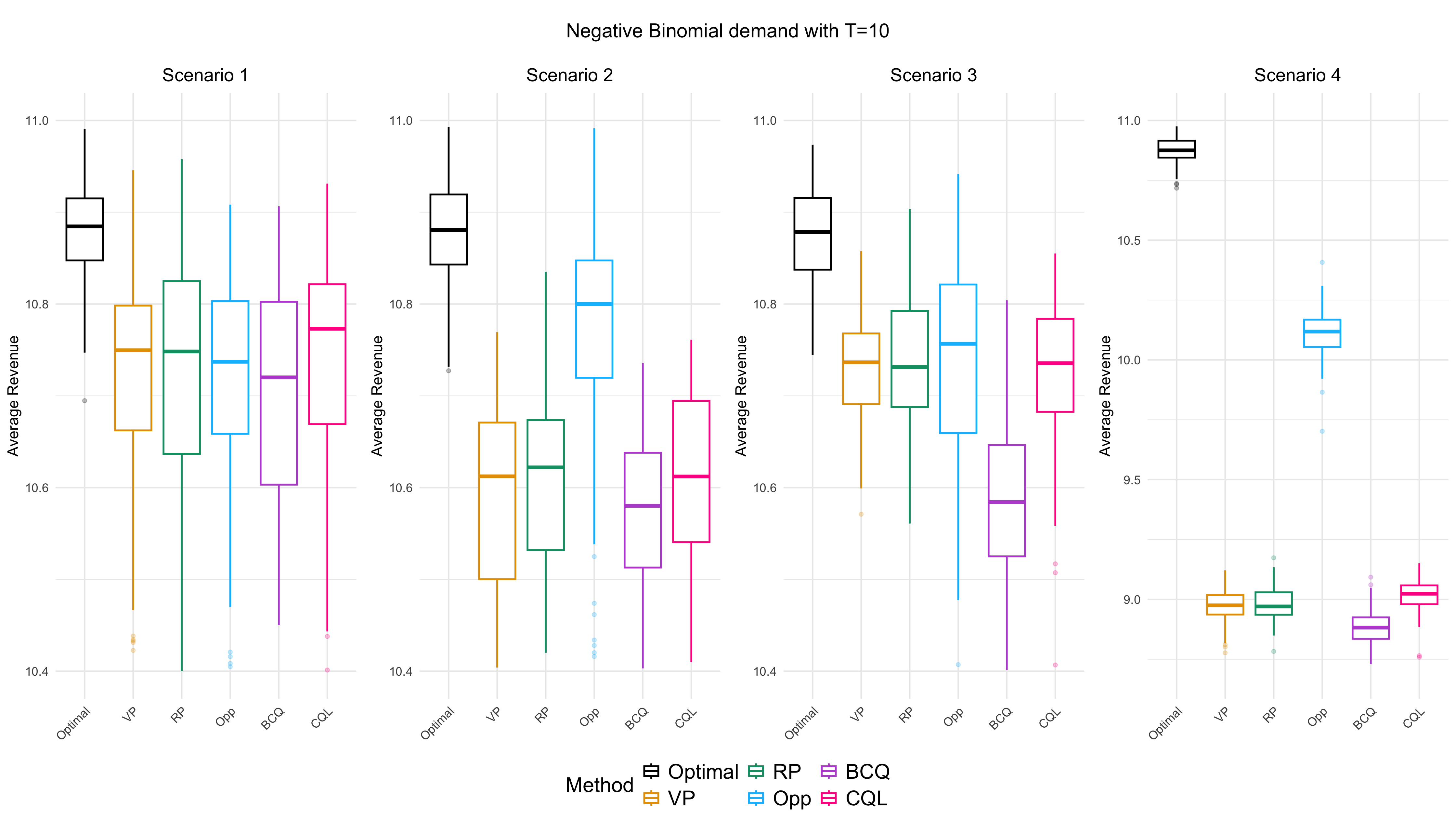}
    \caption{Empirical value functions (average revenues) under four behavior policies with time horizon $T=10$ for the \textbf{\textit{Negative Binomial}} demand model.}
    \label{fig:nb10}
\end{figure}

\begin{figure}[t]
    \centering
    \includegraphics[width=0.9\textwidth]{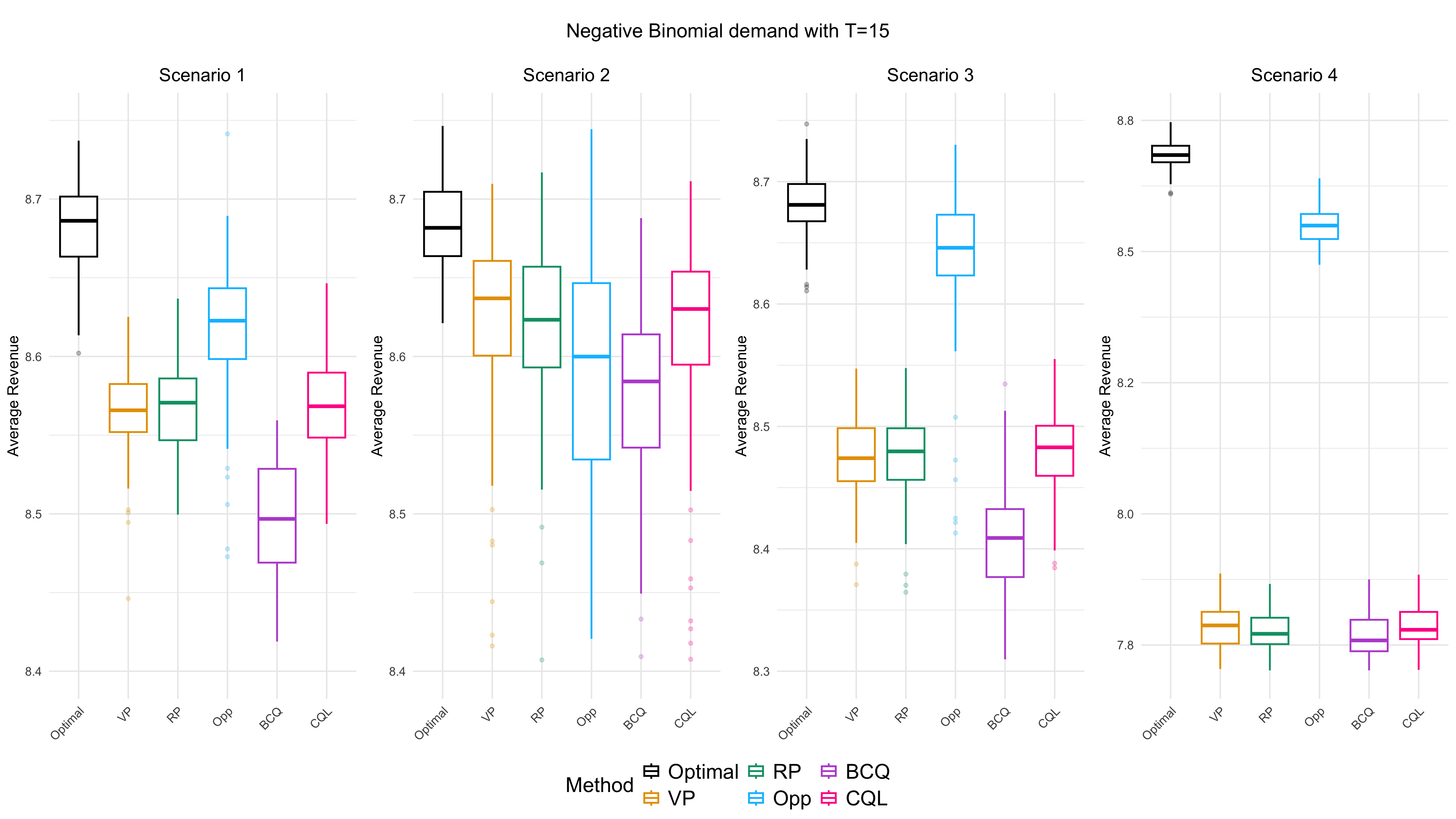}
    \caption{Empirical value functions (average revenues) under four behavior policies with time horizon $T=15$ for the \textbf{\textit{Negative Binomial}} demand model.}
    \label{fig:nb15}
\end{figure}

\begin{figure}[t]
    \centering
    \includegraphics[width=0.9\textwidth]{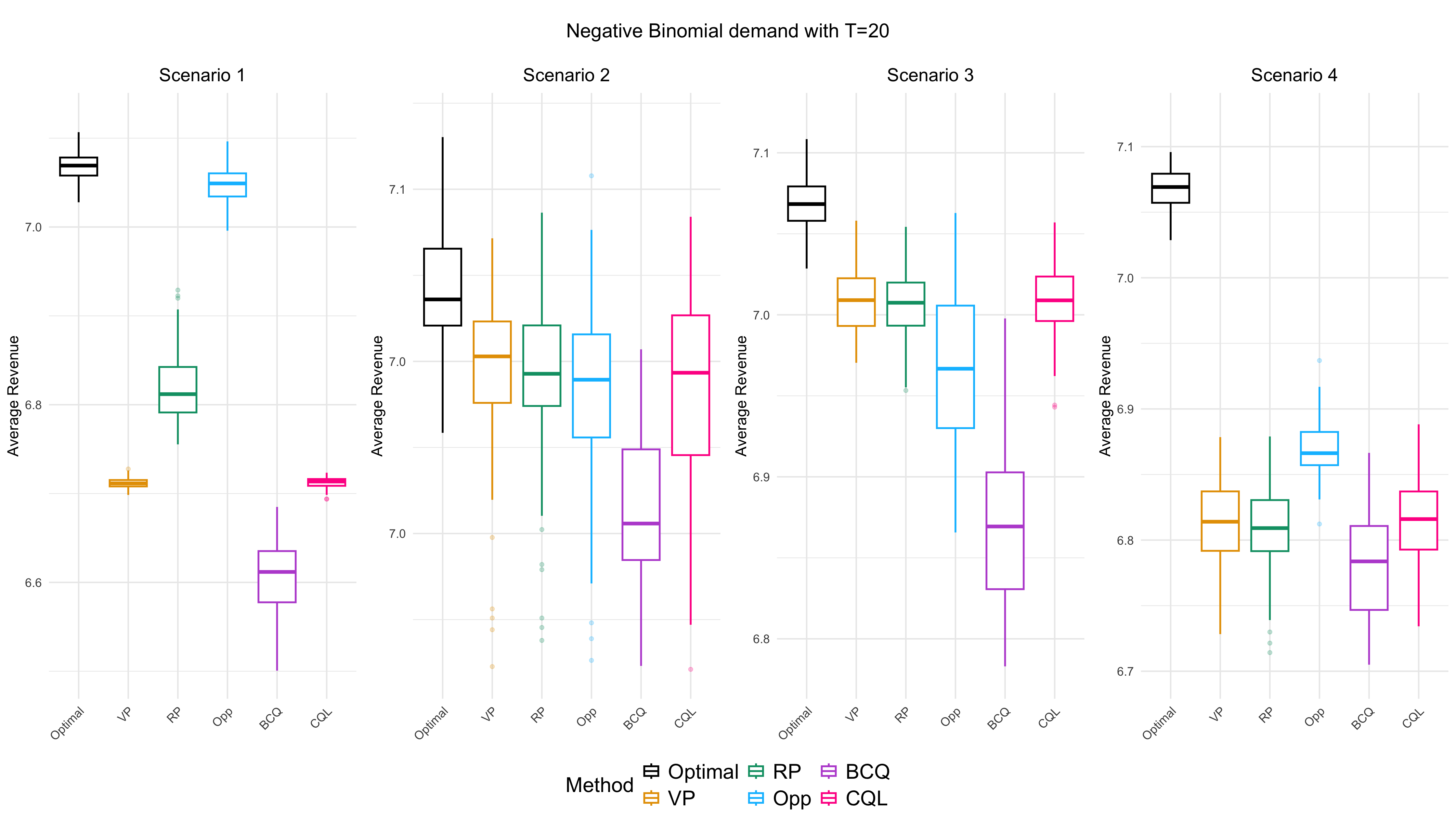}
    \caption{Empirical value functions (average revenues) under four behavior policies with time horizon $T=20$ for the \textbf{\textit{Negative Binomial}} demand model.}
    \label{fig:nb20}
\end{figure}

\begin{figure}[t]
    \centering
    \includegraphics[width=0.9\textwidth]{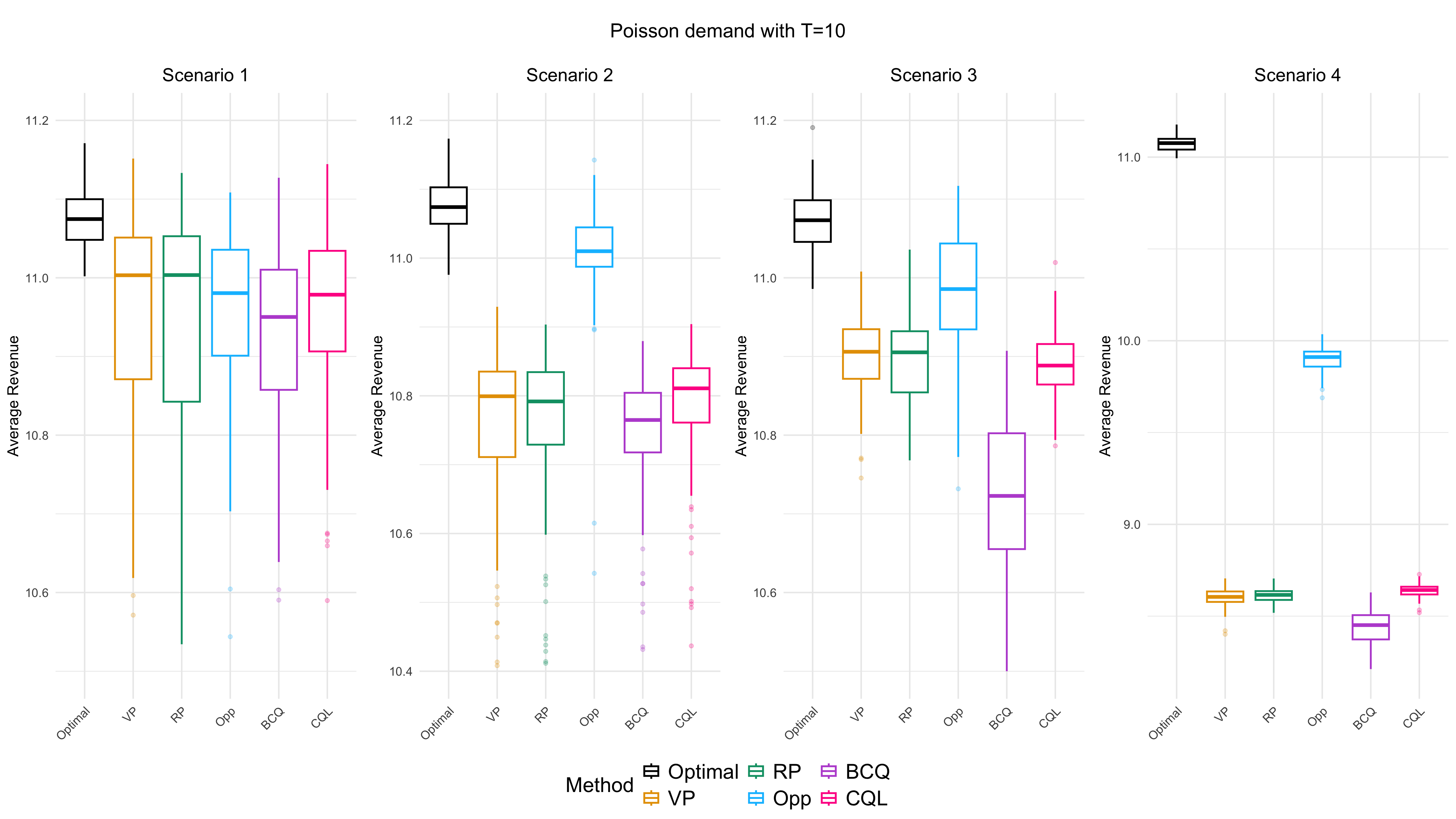}
    \caption{Empirical value functions (average revenues) under four behavior policies with time horizon $T=10$ for the \textbf{\textit{Poisson}} demand model.}
    \label{fig:poisson10}
\end{figure}

\begin{figure}[t]
    \centering
    \includegraphics[width=0.9\textwidth]{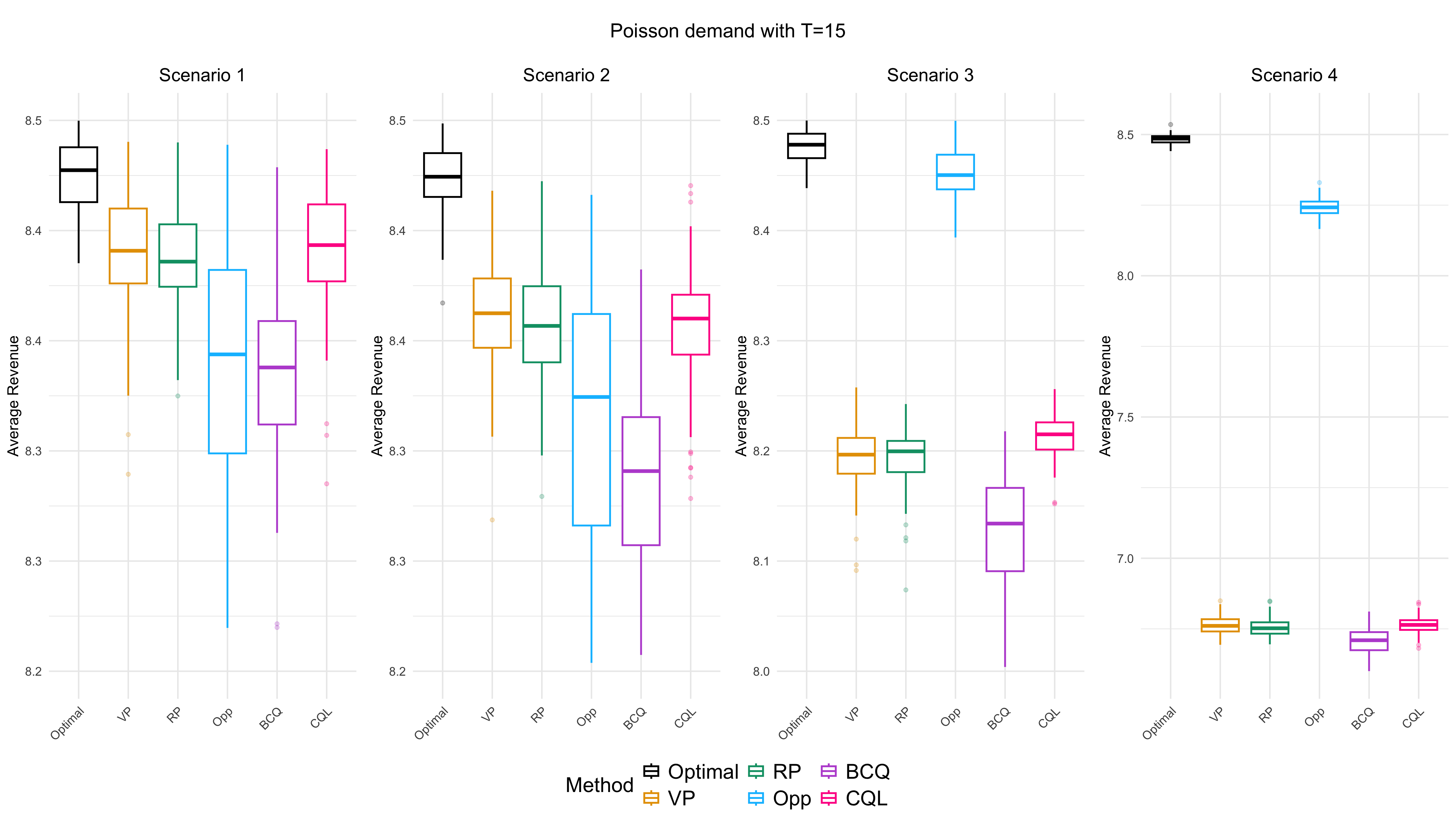}
    \caption{Empirical value functions (average revenues) under four behavior policies with time horizon $T=15$ for the \textbf{\textit{Poisson}} demand model.}
    \label{fig:poisson15}
\end{figure}

\begin{figure}[t]
    \centering
    \includegraphics[width=0.9\textwidth]{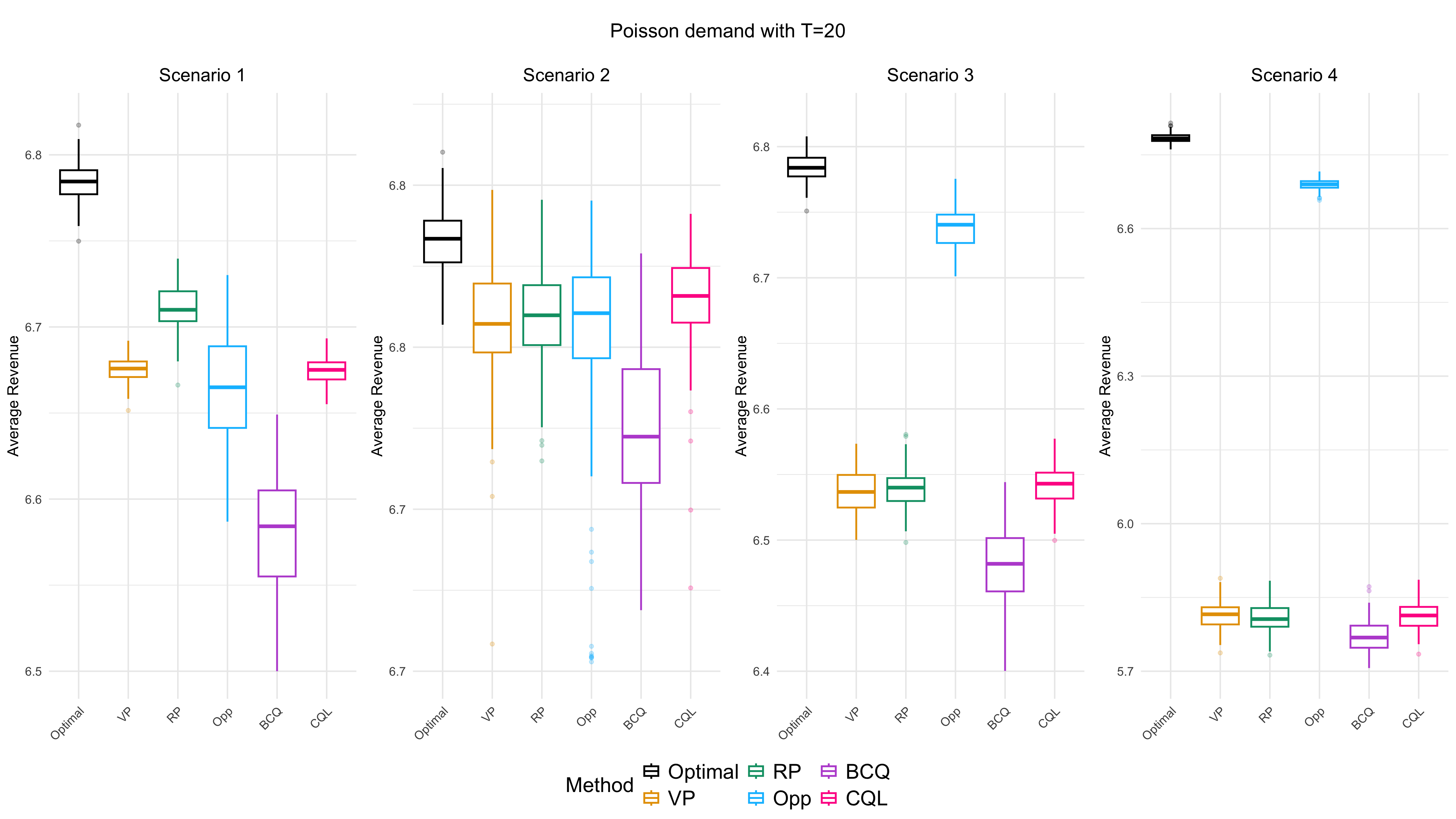}
    \caption{Empirical value functions (average revenues) under four behavior policies with time horizon $T=20$ for the \textbf{\textit{Poisson}} demand model.}
    \label{fig:poisson20}
\end{figure}

Figures \ref{fig:nb10}--\ref{fig:nb20} and \ref{fig:poisson10}--\ref{fig:poisson20} present the empirical value functions (average revenues) of all six methods across four behavior policies for time horizons $T \in \{10,15,20\}$ under the \textbf{\textit{Negative Binomial}} and \textbf{\textit{Poisson}} demand models.

We briefly summarize our main findings. First, the proposed opportunistic method outperforms other methods in most scenarios, especially in the most challenging scenario 4, where the behavior policy is deliberately suboptimal. Moreover, in several settings, the opportunistic method achieves performance comparable to the oracle optimal policy, even when multiple prices are missing from the offline data. For example, this occurs for $T=10$ under behavior policies (ii) and (iii), and for $T=15$ under behavior policy (iii) in the Poisson model, as well as for $T=20$ under behavior policy (i) in the Negative Binomial model. Second, the proposed refined pessimistic method outperforms the vanilla pessimistic approach in several cases, for example, in Scenario 1 when $T=20$ for both the Negative Binomial and Poisson demand models, highlighting the benefit of incorporating our proposed partial identification bounds into the pessimistic strategy.

% \begin{table}[H]
% \centering
% \caption{Empirical average frequency with which missing prices are selected by different methods across five behavior policies, averaged over three time horizons $T=10, 15, 20$, under the Poisson and Negative Binomial demand models.}
% \label{tab:freq}
% \begin{threeparttable}
% \resizebox{0.8\textwidth}{!}{%
% \begin{tabular}{lcccccccccc}
% \toprule
% \multicolumn{1}{c}{} & \multicolumn{5}{c}{Poisson Demand} & \multicolumn{5}{c}{Negative Binomial Demand} \\
% \cmidrule(l{4pt}r{4pt}){2-6} \cmidrule(l{4pt}r{4pt}){7-11}
% Behavior Policy & VP & RP & Opp & BCQ & CQL & VP & RP & Opp & BCQ & CQL \\
% \midrule
% i & 0 & 1 & 1 & 0 & 0 & 0 & 1 & 1 & 0 & 0 \\
% ii & 0 & 0 & 1 & 0 & 0 & 0 & 0 & 0.9 & 0 & 0 \\
% iii & 0 & 0 & 1 & 0 & 0 & 0 & 0 & 1 & 0 & 0 \\
% iv & 0 & 0.1 & 1 & 0 & 0 & 0 & 0.1 & 1 & 0 & 0 \\
% v & 0 & 0 & 1 & 0 & 0 & 0 & 0.01 & 1 & 0 & 0 \\
% \bottomrule
% \end{tabular}%
% }
% \end{threeparttable}
% \end{table}

During the simulations, we also found that the opportunistic approach is able to select missing prices in almost all scenarios. In contrast, the refined pessimistic method can select a missing price only in scenario 1. This observation is fully consistent with our theoretical analysis in Proposition~\ref{prop:pess unable} and Example~\ref{example: pess able}: in scenario 1, the highest price $10$ is missing, and the refined pessimistic approach can select a missing action only when the missing price is the largest  in the candidate price set $\mathcal{A}$. Finally, the vanilla pessimistic method, BCQ, and CQL never select a missing price in any scenario, which is again consistent with their theoretical guarantees and expected behavior.

\subsection{Real Data Analysis} \label{sec:real data}

\ZB{
In this section, we analyze an airline ticket dataset\footnote{The dataset is available at \texttt{https://www.kaggle.com/datasets/dilwong/flightprices}.} from Expedia, focusing on pricing for non-basic economy fares on non-stop flights from New York (JFK) to Charlotte (CLT).

After basic data cleaning, we obtain a dataset with $N=179$ trajectories and $T=10$ time points per trajectory. Each trajectory corresponds to a single flight (i.e., the flight number is fixed within a trajectory), while different trajectories correspond to different flight numbers. Within each trajectory, the $T$ time points correspond to different search dates prior to departure.

We discretize the prices into 10 levels for convenience. For each price level, we report the number of observations (frequency) and the sample-average demand. As shown in Table~\ref{tab:demand} in Appendix, prices 642, 757, 901, and 1272 are relatively under-represented in the offline dataset; in particular, price 757 appears only twice. Accordingly, we treat price 757 as unobserved when evaluating all the methods. Finally, Table~\ref{tab:demand} also suggests that the monotonicity assumption for the mean demand is likely to hold.

\begin{figure}[H]
    \centering
    \includegraphics[width=0.8\textwidth]{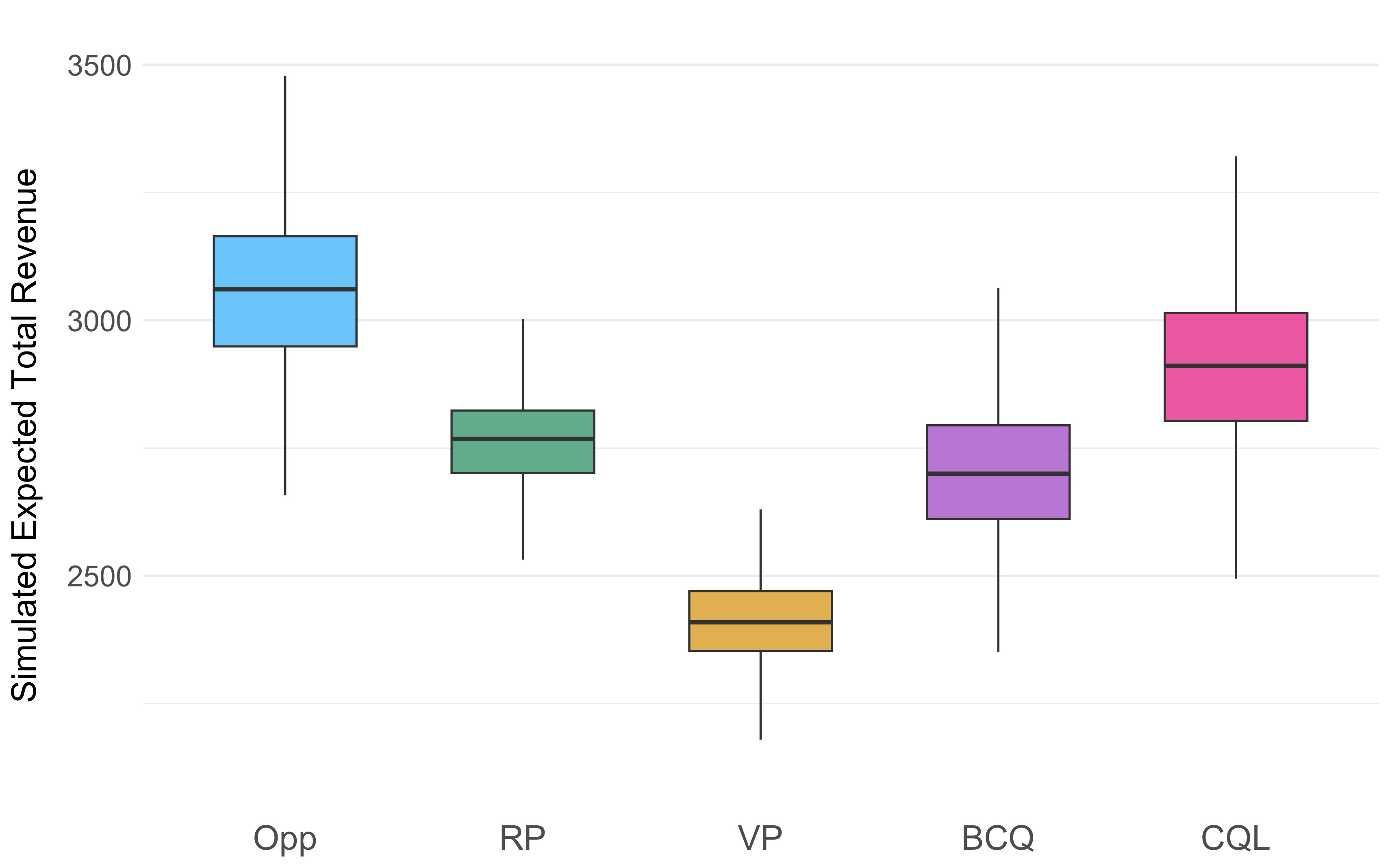}
    \caption{Empirical total revenue across five methods based on 500 Monte Carlo simulations.}
    \label{fig:airline}
\end{figure}

We apply the proposed refined pessimistic and opportunistic approaches to this dataset, and compare them with the vanilla pessimistic approach as well as BCQ and CQL. To evaluate performance, we use Monte Carlo simulation to calculate empirical total revenues under a Poisson demand model calibrated from the offline data. As shown in Figure~\ref{fig:airline}, the proposed opportunistic method achieves the best performance. The refined pessimistic approach ranks third, performing slightly worse than CQL. As before, the vanilla pessimistic approach has the lowest empirical revenue.
Tables~\ref{tab:opp_pricing} and~\ref{tab:cql_pricing} in the Appendix report the estimated optimal prices across time points and inventory levels under the opportunistic approach and CQL, respectively. The comparison shows that the opportunistic approach can select the unobserved price 757 for some states, whereas CQL cannot.
}

% \ZL{Only presenting expected total revenue in the real data analysis is not interesting. We should try our best to get some meaningful insight from the real data analysis.}

\section{Discussion} \label{sec:discussion}

\textit{\textbf{Pessimism versus Opportunism}}. In statistical decision theory, no single strategy uniformly outperforms others, as each approach—whether pessimistic, or opportunistic—brings unique strengths and limitations to decision-making under uncertainty. Pessimism emphasizes robustness against worst-case scenarios, while opportunism aims to maximize expected utility based on available information and opportunities. Each strategy's effectiveness depends on the specific context and risk preferences involved. For small or conservative companies, especially those with limited financial resources, the pessimistic approach might be favored. These companies prioritize avoiding financial losses that could threaten their survival. By focusing on minimizing potential risks and preparing for worst-case scenarios, they aim to maintain stability and ensure continued operations without facing the severe consequences of bankruptcy. 

On the other hand, well-established companies or those with ambitious growth objectives may lean toward the opportunistic approach. These companies have more resources and capacity to withstand risks. They are willing to take calculated chances to maximize potential gains and seize opportunities that align with their growth strategies. This approach allows them to capitalize on favorable market conditions, innovations, or strategic moves that could significantly enhance their market position or profitability. In both cases, the choice of approach reflects not only the company's financial standing and risk tolerance but also its broader strategic goals and outlook for managing uncertainty in the business environment.

\textit{\textbf{Static Interval}}. Unlike classical confidence bounds, the partial identification interval for an unobserved price does not necessarily shrink to the true value as the sample size increases. Consequently, the proposed pessimistic and opportunistic approaches may impose excessive penalties on unobserved but potentially optimal prices. 

That said, given the fundamental challenge posed by unobserved prices, the results in Theorems~\ref{thm: regret 2} and \ref{theorem:regret3} provide the sharpest attainable guarantees in our setting. Any further improvement therefore requires additional structure. We briefly describe several structural assumptions and explain how they can be leveraged to strengthen our results in Section \ref{sec:structure} of the Appendix.

	\textit{\textbf{Feature-based Learning and Continuous Pricing
}}. In this article, we assume that demand depends only on time and price. Future work could further consider the \textbf{\textit{feature-based}} demand-learning problem \citep{broder2012dynamic,keskin2014dynamic,javanmard2019dynamic,ban2021personalized,bastani2022meta,qi2022offline}. In addition, this paper considers only a discrete price space, leaving the more challenging continuous-pricing setting for future investigation. Another interesting direction is how to efficiently combine offline and online learning to achieve sustainable growth and profitability.

% \textit{\textbf{Discrete prices}}.
% We focus on discrete prices in this work; this setting is not only common in the pricing literature \citep[see, e.g.,][]{gallego1997multiproduct,manchiraju2022multiproduct}, but also reflects many real-world applications in which firms set prices from a finite set. This is further supported by the airline pricing data analyzed in Section \ref{sec:real data}, which exhibits a discrete set of observed prices. Moreover, even when prices are conceptually continuous, discrete approximations are routinely used in practice and computation; our framework can be interpreted as operating on such a discretized price grid. That said, continuous pricing remains important, and extending our framework to handle continuous prices is a promising direction for future research.

\bibliographystyle{apalike}
\bibliography{references}

\newpage

\begin{center}
    {\LARGE\bf Supplemental Materials for \\``A Tale of Two Cities: Pessimism and Opportunism in Offline Dynamic Pricing''}
\end{center}

\section{Extensions and Additional Simulations} \label{sec:extension}

\subsection{Censored Demand}
Earlier analyses assume that demand is fully observed. This assumption holds in some settings, such as online or digital goods markets. However, in some cases, the true demand cannot be observed when the inventory level is lower than demand. In such cases, only $\min(X, D)$ is observed—a phenomenon known as demand censoring \citep{huh2011adaptive,ban2020confidence,qi2022offline,bu2023offline}. Under this situation, one can still apply a partial identification approach to address the unobserved price. However, the procedure for constructing the partial identification set should be modified to account for demand censoring. Specifically, tools that incorporate the censoring mechanism should be used, and we propose employing the Kaplan–Meier estimator \citep{kaplan1958nonparametric} to construct such an interval. We remark that the Kaplan–Meier estimator is commonly adopted to handle censored demand in pricing problems, see, e.g., \citet{huh2011adaptive,bu2023offline}. For a thorough treatment of the Kaplan–Meier estimator and survival analysis, we refer readers to \citet{kalbfleisch2002statistical}.

We now describe how to construct the partial identification interval based on the Kaplan–Meier estimator. For any price $a \in \mathcal{A}_t^{\mathcal{D}_N}$, the Kaplan–Meier estimator for the demand CDF $F_t(d|a)$ is  \begin{gather*}
    \widehat F^{\text{KM}}_t(d|a)= 1-\prod_{i: A_{i,t}=a} \left(1- \frac{\mathds{1}(D_{i,t}=d,D_{i,t} <X_{i,t})}{\mathds{1}(D_{i,t}\geq d)}\right).
\end{gather*} Then, following the same procedure as in Equation~\eqref{eq:CI}, we can construct the interval $\left[\widehat F_t^L(d|a), \widehat F_t^U(d|a)\right]$ based on $\widehat F_t^l$ and $\widehat F_t^u$. Note that here $\widehat F_t^l$ and $\widehat F_t^u$ should be modified accordingly. Specifically, unlike the previous analysis, we now redefine
$\widehat F_t^l(d|a') = \widehat F_t^{\text{KM}}(d|a') - \delta_t(a')$ and $\widehat F_t^u(d|a') = \widehat F_t^{\text{KM}}(d|a') + \delta_t(a')$. Finally, the uncertainty quantifier $\delta_t(a)$ can still be chosen on the order of $c\sqrt{\log N / N_t(a)}$ for some positive constant $c$. This choice is justified by the  concentration inequalities for the Kaplan–Meier estimator, see, e.g., \citet{foldes1981strong, bitouze1999dvoretzky}. After obtaining such a partial identification interval, one can directly apply Algorithms \ref{alg:2} and \ref{alg:3} for the refined pessimistic and opportunistic approaches, respectively.

\subsection{Piecewise Monotonicity Assumption}

The proposed approach relies on the monotonicity assumption. As noted earlier, this assumption is relatively mild, plausible, and it holds in a wide range of applications. Nonetheless, it may not hold for certain products, such as Veblen goods. In this section, we relax the monotonicity assumption to piecewise monotonicity, allowing for a more general and flexible framework. We now formally introduce the following piecewise monotonicity assumption.

\begin{assumption}[Piecewise Monotonicity] \label{assumption:picewise} There exists a known price threshold $a_p$ such that, for any $t$ and any $d$, 
\begin{align*}
    F_t(d \mid a) \le F_t(d \mid a^+), \quad \text{whenever } a < a^+ \le a_p, \\
 \mbox{and }
F_t(d \mid a) \ge F_t(d \mid a^+), \quad \text{whenever } a_p \le a < a^+.
\end{align*}
\end{assumption}

Assumption~\ref{assumption:picewise} relaxes the global monotonicity condition in Assumption~\ref{assumption:mono} by allowing the conditional demand distribution to be monotone in opposite directions below and above a known price threshold $a_p$. When $a_p = a_{\max}$, Assumption~\ref{assumption:picewise} reduces to Assumption~\ref{assumption:mono}. This assumption captures settings where demand decreases with price at lower price levels, as in standard goods, but increases beyond $a_p$ due to prestige effects typical of luxury or Veblen goods. Please refer to Figure \ref{fig:piecewise-mono} for a further demonstration.

\begin{figure}[H]
    \centering
    \begin{tikzpicture}
        \begin{axis}[
            width=0.85\linewidth,
            height=5cm,
            xmin=0, xmax=10,
            ymin=0, ymax=1,
            xlabel={Price $a$},
            ylabel={$F_t(d\mid a)$},
            axis lines=left,
            ticks=none,
            clip=false,
        ]
            % Schematic: increasing up to a_p then decreasing
            \addplot[very thick, blue, smooth] coordinates {
                (0.5,0.15)
                (2,0.30)
                (4,0.55)
                (5,0.70)
                (6,0.62)
                (8,0.40)
                (9.5,0.25)
            };

            % Mark the threshold a_p
            \addplot[dashed, gray] coordinates {(5,0) (5,1)};
            \node[anchor=south, gray] at (axis cs:5,1.02) {$a_p$};

            % Annotations
            \node[anchor=west, blue] at (axis cs:0.6,0.86) {non-decreasing for $a\le a_p$};
            \node[anchor=west, blue] at (axis cs:5.8,0.12) {non-increasing for $a\ge a_p$};
        \end{axis}
    \end{tikzpicture}
    \caption{Schematic illustration of Assumption~\ref{assumption:picewise}: for any fixed $(t,d)$, the conditional CDF $F_t(d\mid a)$ is monotone non-decreasing in price up to the threshold $a_p$, and monotone non-increasing beyond $a_p$.}
    \label{fig:piecewise-mono}
\end{figure}

We now describe how to construct the partial identification interval based on this assumption. For simplicity, we assume $a_p \in \mathcal{A}_t^{\mathcal{D}_N}$, $\forall t$. 
For any price $a < a_p$, one can directly apply the procedure discussed in the earlier section to construct the interval for $F_t(\cdot|a)$. For $a > a_p$, however, the previously introduced interval $\left[\widehat F_t^L(d \mid a),\, \widehat F_t^U(d \mid a)\right]$ for $F_t(d \mid a)$ should be modified to
:  \begin{align*} 
\begin{split}
   & \widehat F_t^L(d|a)= \begin{cases}
        \max_{a'> a}\widehat F_t^l(d|a'), & \mbox{ if } \left\{\widehat F_t^l(d|a')\right\}_{a'> a} \mbox{ is non-empty },\\
        \varepsilon,& \mbox{ otherwise}.
    \end{cases} \\
      & \widehat F_t^U(d|a)= \begin{cases}
        \min_{a_p \leq a'< a}\widehat F_t^u(d|a'), & \mbox{ if } \left\{\widehat F_t^u(d|a')\right\}_{a_p \leq a'< a} \mbox{ is non-empty },\\
        1-\varepsilon,& \mbox{ otherwise},
    \end{cases}
    \end{split}.
\end{align*} Finally, the interval corresponding to the cut-off price $a_p$ requires special treatment, since for any price $a \neq a_p$, and for all $t$ and $d$, we have $F_t(d \mid a) \le F_t(d \mid a_p)$. Thus, only the lower bound for $F_t(d \mid a_p)$ can be refined, and this result in the following interval for $F_t(d \mid a_p)$: \begin{gather*}
    \widehat F_t^L(d|a_p)= \max_a \widehat F_t^l(d|a), \mbox{ and } \widehat F_t^U(d|a_p)= \widehat F_t^u(d|a_p).
\end{gather*}
After obtaining such a partial identification interval, one can directly apply Algorithms \ref{alg:2} and \ref{alg:3} for the refined pessimistic and opportunistic approaches, respectively.

\subsection{Additional Structural Assumptions} \label{sec:structure}

To tighten identification and improve estimation of $Q_t(x,a)$, one can impose several structurally motivated constraints:  \\
(a) Lipschitz continuity: there exists a positive constant $L$ such that $|Q_t(x,a)-Q_t(x,a')|\le L|a-a'|$ for all $t$, $x$, $a$, and $a'$;  \\
% (b) Bounded curvature: there exists a positive constant $L$ such that 
% \[
% |Q_t(x,a^+)-2Q_t(x,a)+Q_t(x,a^-)|\le L,
% \]
% for any $a^+>a>a^-$, $t$, and $x$;\\ 
(b) Temporal stability: for some tolerance $\delta\ge 0$, and for all $t$, $a$, and $d$,
\[
|F_t(d\mid a)-F_{t+1}(d\mid a)|\le \delta.
\]

Collectively, these assumptions impose mild smoothness and stability, which regularize estimation and tighten the resulting identification bounds. For example, under Assumption~(a) Lipschitz continuity, for an unobserved price $a$ at time $t$, we can still construct a confidence bound for $Q_t(x,a)$:
\begin{gather*}
    \Bigl[Q_t^L(x,a),Q_t^U(x,a)
    \Bigr],
\end{gather*} 
where \begin{align*}
 Q_t^L(x,a)=\max\bigl[\widehat Q_t(x,a_t^l)-\Delta_t^a(a_t^l),\;\widehat Q_t(x,a_t^u)-\Delta_t^a(a_t^u)\bigr], \\
 Q_t^U(x,a)=\min\bigl[\widehat Q_t(x,a_t^l)+\Delta_t^a(a_t^l),\;\widehat Q_t(x,a_t^u)+\Delta_t^a(a_t^u)\bigr], 
\end{align*} and  $\Delta_t^a(a')=L|a-a'|+\delta_t(a')$.
Based on this confidence bound, we can apply the proposed opportunistic and pessimistic approaches to obtain an estimated policy.

\subsection{Sensitivity Analysis}

We further investigate the robustness of the proposed method when the monotonicity assumption~\ref{assumption:mono} is violated. We again consider a Negative Binomial demand model
\[
D_{i,t} \mid A_{i,t} = a \sim \mathrm{NB}(\eta_t(a), 10),
\]
and we consider two mean functions.

\begin{figure}[H]
    \centering
    \includegraphics[width=\textwidth]{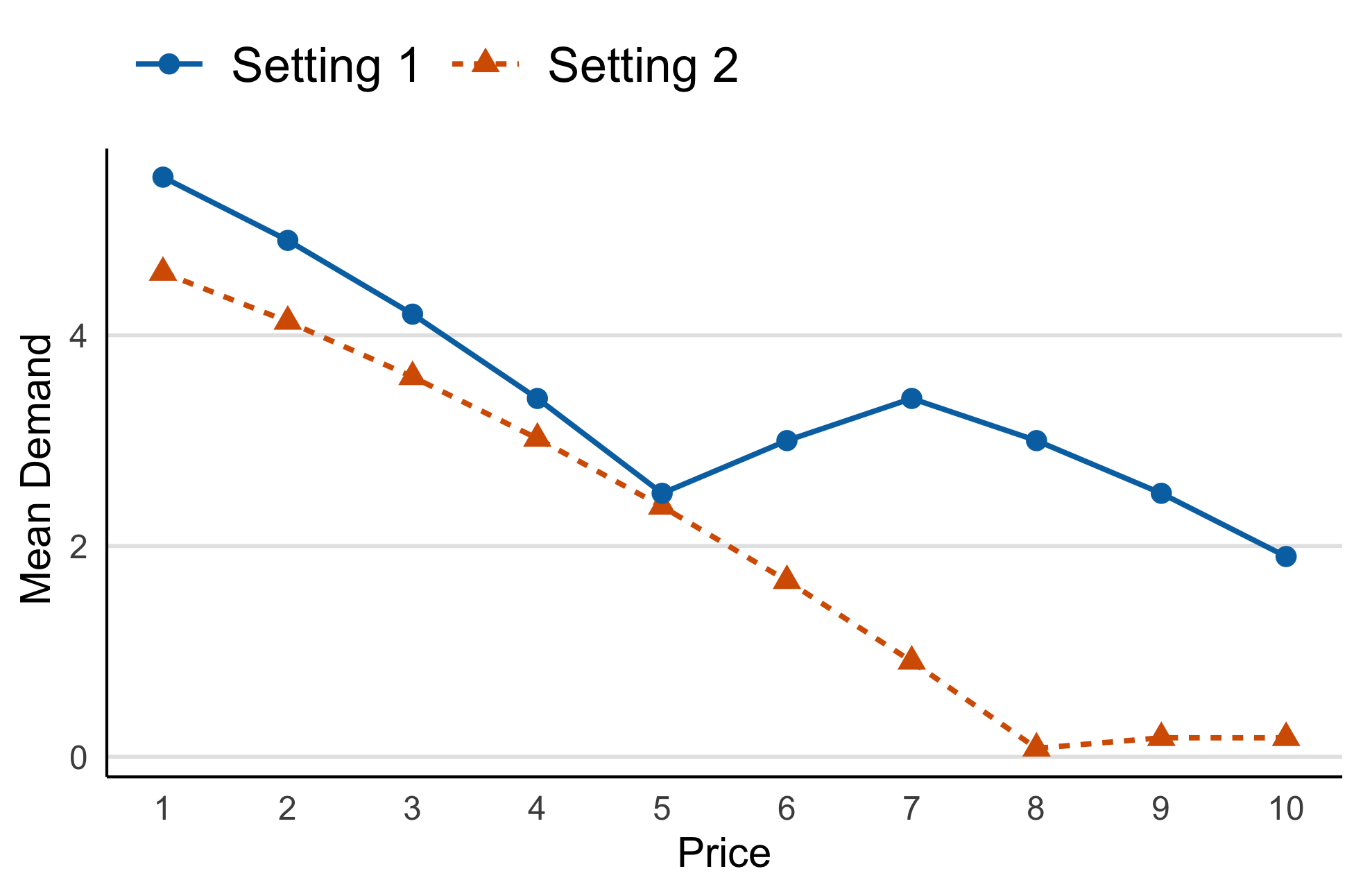}
    \caption{Mean demand curve across two settings.}
    \label{fig:sens_demand}
\end{figure}

In the first setting,
\[
\eta_t(a)
=
\begin{cases}
6 - 0.45a - 0.05a^2, & a \le 5,\\[6pt]
\eta_t(5) + 0.55(a-5) - 0.05(a-5)^2, & 5 < a \le 7,\\[6pt]
\eta_t(7) - 0.35(a-7) - 0.05(a-7)^2, & a > 7,
\end{cases}
\quad \forall t.
\]
In this case, demand decreases with price for $a\le5$, increases over the intermediate range $5<a\le7$, and decreases again for $a>7$.

In the second setting,
\[
\eta_t(a)
=
\begin{cases}
\dfrac{20-1.5a-0.12a^2}{4}, & a \le 8,\\[8pt]
\eta_t(8) + 0.1, & a>8,
\end{cases}
\quad \forall t.
\]
Under this specification, $\eta_t(a)$ decreases with $a$ for $a\le 8$ but exhibits an upward shift for $a>8$. Hence, the mean demand is not globally monotone in price, and Assumption~\ref{assumption:mono} does not hold. See Figure \ref{fig:sens_demand} for the mean demand curves under the two settings. For both settings, we use the same four behavior policies as in the main simulation study. In addition, we fix the initial inventory level at $X_{i,1}=15$ and set the time horizon to $T=10$.

\begin{figure}[H]
    \centering
    \includegraphics[width=\textwidth]{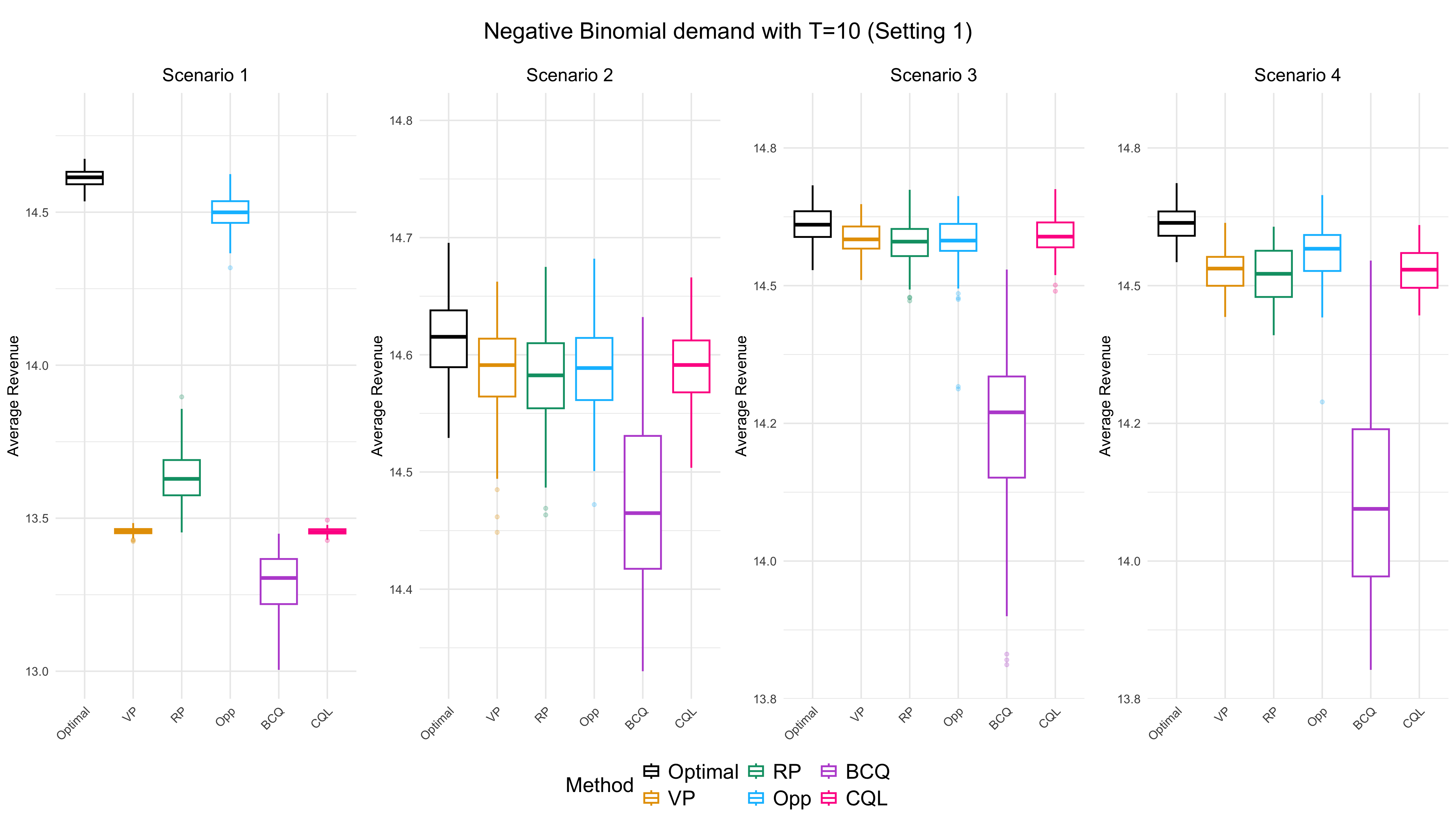}
    \caption{Sensitivity analysis: Empirical value functions (average revenues) under four behavior policies with time horizon $T=10$ for the \textbf{\textit{Negative Binomial}} demand model under Setting 1.}
    \label{fig:sens1}
\end{figure}

The simulation results are summarized in Figures \ref{fig:sens1} and \ref{fig:sens2} for Settings 1 and 2, respectively. In Setting 1, the opportunistic approach still outperforms other methods across all scenarios, and the refined pessimistic method also performs well. This suggests that the opportunistic approach can still be effective even when monotonicity is violated. In Setting 2, unlike in the monotone setting, the opportunistic approach no longer consistently outperforms competing methods across scenarios: it outperforms other methods only in Scenario 3 and is outperformed in Scenario 4. However, the proposed refined pessimistic method remains robust. This indicates that violations of monotonicity can slightly affect the performance of the opportunistic approach.

Setting 2 affects the proposed opportunistic approach more than Setting 1, possibly because the violation of monotonicity in Setting 1 does not affect the optimal price, whereas in Setting 2, the optimal price is located in the non-monotone region. Overall, these results indicate that while the monotonicity assumption is important for the theoretical guarantees of our methods, the methods can still perform reasonably well in practice even when this assumption is violated to some extent.

\begin{figure}[H]
    \centering
    \includegraphics[width=\textwidth]{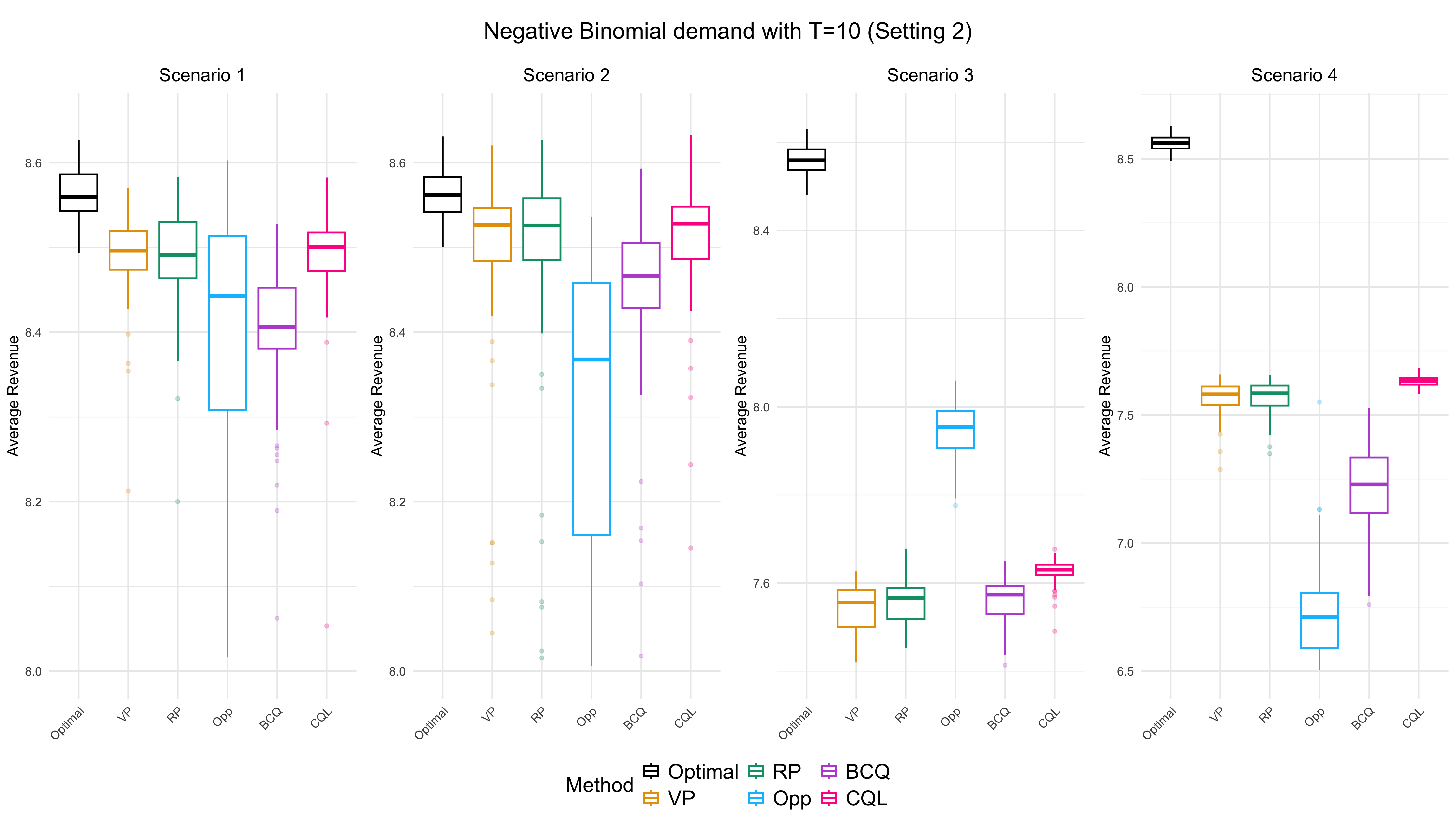}
    \caption{Sensitivity analysis: Empirical value functions (average revenues) under four behavior policies with time horizon $T=10$ for the \textbf{\textit{Negative Binomial}} demand model under Setting 2.}
    \label{fig:sens2}
\end{figure}

\subsection{Additional Details for Airline Tickets Data}

\begin{table}[H]
\centering
\caption{Observed prices with their frequencies and corresponding average demand.}
\label{tab:demand}
\begin{tabular}{ccc}
\toprule
Price & Frequency & Average Demand \\
\midrule
149  & 86  & 0.69 \\
189  & 276 & 0.55 \\
224  & 425 & 0.60 \\
255  & 463 & 0.67 \\
296  & 186 & 0.60 \\
383  & 305 & 0.56 \\
642  & 18  & 0.28 \\
757  & 2   & 0.50  \\
901  & 13  & 0.23 \\
1272 & 16  & 0.25  \\
\bottomrule
\end{tabular}
\end{table} 

\begin{table}[H]
\centering
\caption{Estimated optimal prices over time under the opportunistic approach.}
\label{tab:opp_pricing}
\begin{tabular}{c|ccccccccc}
\toprule
& \multicolumn{9}{c}{Inventory level} \\
\cmidrule(lr){2-10}
Time & 1 & 2 & 3 & 4 & 5 & 6 & 7 & 8 & 9 \\
\midrule
1  & 1272 & 1272 & 1272 & 1272 & 1272 & 1272 & 1272 & 1272 & 1272 \\
2  & 1272 & 1272 & 1272 & 1272 & 1272 & 1272 & 1272 & 1272 & 1272 \\
3  & 1272 & 1272 & 1272 & 1272 & 1272 & 1272 & 1272 & 1272 & 1272 \\
4  & 1272 & 1272 & 1272 & 1272 & 1272 & 1272 & 1272 & 1272 & 1272 \\
5  & 1272 & 1272 & 1272 & 1272 & 1272 & 1272 & 1272 & 1272 & 1272 \\
6  & 1272 & 1272 & 1272 & 1272 & 1272 & 1272 & 1272 & 1272 & 1272 \\
7  & 1272 & 1272 & 1272 & 1272 & 1272 & 1272 & 1272 & 1272 & 1272 \\
8  & 757  & 1272 & 1272 & 1272 & 1272 & 1272 & 1272 & 1272 & 1272 \\
9  & 757  & 1272 & 383  & 1272 & 1272 & 1272 & 1272 & 1272 & 1272 \\
10 & 383  & 383  & 1272 & 1272 & 1272 & 1272 & 1272 & 1272 & 1272 \\
\bottomrule
\end{tabular}
\end{table}

\begin{table}[H]
\centering
\caption{Estimated optimal prices over time under CQL.}
\label{tab:cql_pricing}
\begin{tabular}{c|ccccccccc}
\toprule
& \multicolumn{9}{c}{Inventory level} \\
\cmidrule(lr){2-10}
Time & 1 & 2 & 3 & 4 & 5 & 6 & 7 & 8 & 9 \\
\midrule
1  & 1272 & 1272 & 1272 & 1272 & 1272 & 1272 & 1272 & 1272 & 1272 \\
2  & 901  & 1272 & 1272 & 1272 & 1272 & 1272 & 1272 & 383  & 1272 \\
3  & 1272 & 1272 & 1272 & 1272 & 1272 & 1272 & 1272 & 383  & 1272 \\
4  & 901  & 1272 & 1272 & 1272 & 1272 & 1272 & 1272 & 383  & 1272 \\
5  & 901  & 901  & 1272 & 1272 & 1272 & 1272 & 383  & 1272 & 1272 \\
6  & 901  & 901  & 1272 & 1272 & 1272 & 1272 & 383  & 1272 & 1272 \\
7  & 383  & 1272 & 1272 & 1272 & 1272 & 383  & 1272 & 1272 & 1272 \\
8  & 901  & 1272 & 1272 & 1272 & 901  & 1272 & 1272 & 1272 & 1272 \\
9  & 1272 & 383  & 1272 & 1272 & 1272 & 1272 & 1272 & 1272 & 1272 \\
10 & 1272 & 383  & 1272 & 1272 & 1272 & 1272 & 1272 & 1272 & 1272 \\
\bottomrule
\end{tabular}
\end{table}

\section{Proofs and Auxiliary Lemmas}

Throughout, we use $c$ and $C$ to denote a generic constant that can vary from line to line. Moreover, to ease the notation, sometimes we write the estimated action $\widehat a_t^{\mathbf{pess}}(x)$ and $\widehat a_t^{\mathbf{opp}}(x)$ as $\widehat a_t^{\mathbf{pess}}$ and $\widehat a_t^{\mathbf{opp}}$, respectively, when there is no confusion. In what follows, $\widehat \pi$, $\widehat V_t(x)$ and $\widehat Q_t(x,a)$ all represent generic estimators, which can be obtained using any algorithms.

\subsection{Proof of Proposition \ref{prop: bellman}} \label{appendix:prop}

\begin{proof}

By the Bellman optimality equation \eqref{eq: bellman}, we have \begin{align*}
    Q_t^*(x,a)=& \E \left[  R_t+ V^*_{t+1}(X_{t+1})   |X_t=x,A_t=a \right] \\
    =& \underbrace{\E \left( R_t  |X_t=x,A_t=a \right) }_{\mbox{expected immediate reward}} + \underbrace{\E \left[  V^*_{t+1}(X_{t+1})|X_t=x,A_t=a \right]}_{\mbox{expected future cumulative reward}}.
\end{align*} 
We first derive the term immediate rewards. Denote $M_t$ as the random variable of the number of sold items in time $t$, and $D_t $ the demand for price $a$, so that $M_t=\min(X_t, D_t)$. Thus the reward $R_t= A_t M_t$, and the reward function \begin{align*}
    \E \left( R_t  |X_t=x,A_t=a \right)=a\E \left(  M_t|X_t=x,A_t=a \right).
\end{align*} 

It remains to calculate $\E \left(  M_t|X_t=x,A_t=a \right)$,  \begin{align*}
    &\E \left(  M_t|X_t=x,A_t=a \right)\\
     = & \; \E(D_t|a, D_t \leq x ) \p(D_t \leq x|a) +x\p(D_t \geq x+1|a)\\
     = & \; \sum_{d=0}^x d \; \p(D_t=d|a)   \; \p^{-1}(D_t \leq x|a)  \p(D_t \leq x|a)+x(1-F_t(x|a))\\
     = & \; \sum_{d=0}^x d \; \p(D_t=d|a) +x(1-F_t(x|a))=x-\sum_{d=0}^{x-1}F_t(d|a)
\end{align*} 
Similarly, we have \begin{align*}
   & \E \left[  V^*_{t+1}(X_{t+1})|X_t=x,A_t=a \right]=\sum_{d=0}^{x-1} V_{t+1}^*(x-d) \p(D_t=d|a)\\
     = & \; \sum_{d=0}^{x-1} \left[V_{t+1}^*(x-d)-V_{t+1}^*(x-d-1)\right] F_t(d|a).
\end{align*} Here, the summation goes up to $x-1$, using the fact that $V_t(0)=0$ for any $t$.

Finally, we have \begin{align*}
    Q_t^*(x,a)=a\left[x-\sum_{d=0}^{x-1}F_t(d|a)\right]+\sum_{d=0}^{x-1} \left[V_{t+1}^*(x-d)-V_{t+1}^*(x-d-1)\right] F_t(d|a).
\end{align*} \end{proof}

\subsection{Proposition \ref{prop:pess} and its Proof}

\begin{prop} \label{prop:pess}
    There exists some universal constant $c$, such that with probability at least $1-\sum_{t=1}^T \left( \frac{|\mathcal{A}_t^{\mathcal{D}_N}|}{N}+\kappa_t \right)$, \begin{gather*}
        \widehat{Q}_t^{\mathbf{pess}}(x,a)\leq Q_t^*(x,a), \mbox{ for all } t, x, \mbox{ and } a.
    \end{gather*}
\end{prop}

\begin{proof}

Note that the above inequality holds for all $a \in \mathcal{A}$, because we are able to construct a confidence interval $\Omega_t(x,a)$ for all $a \in \mathcal{A}$. 

We use mathematical induction to prove this argument. At time point $T+1$, by definition, we have $\widehat{V}_{T+1}^{\mathbf{pess}}(x)= V_{T+1}^*(x)=0$. Then at time $T$, it follows that \begin{align*}
    & \widehat{Q}_T^{\mathbf{pess}}(x,a)=\min_{F \in \Omega_t(x,a)} (\B_T^F \widehat{V}_{T+1})(x,a) =\min_{F \in \Omega_t(x,a)} (\B_T^F V^*_{T+1})(x,a) \\
    \leq \; & (\B_TV^*_{T+1})(x,a)={Q}_T^*(x,a),
\end{align*} where the inequality holds because, for all $a$, $d,$ and $F_t(d|a) \in \Omega_t(x,a)$ with probability at least $1-\sum_{t=1}^T  \frac{|\mathcal{A}_t^{\mathcal{D}_N}|}{N}$.

Now assume $\widehat{Q}_{t+1}^{\mathbf{pess}}(x,a)\leq Q_{t+1}^*(x,a),$ it remains to show that 
$\widehat{Q}_{t}^{\mathbf{pess}}(x,a)\leq Q_{t}^*(x,a).$ Since $\widehat{Q}_{t+1}^{\mathbf{pess}}(x,a)\leq Q_{t+1}^*(x,a),$ we also have $\widehat{V}_{t+1}^{\mathbf{pess}}(x)\leq V_{t+1}^*(x),$ it then follows that \begin{align*}
    & \widehat{Q}_{t}^{\mathbf{pess}}(x,a)=\min_{F \in \Omega_{t}(a)} (\B_{t}^F \widehat{V}^{\mathbf{pess}}_{t+1})(x,a) \\
    \leq \; & \min_{F \in \Omega_{t}(a)} (\B_{t}^F V^*_{t+1})(x,a) \\
    \leq \; & (\B_{t}V^*_{t+1})(x,a)={Q}_{t}^*(x,a),
\end{align*} where the first inequality holds 
by using the property of the Bellman operator, i.e., \begin{align*}
    \widehat{V}_{t+1}^{\mathbf{pess}}(x)\leq V_{t+1}^*(x) \implies (\B_{t}^F \widehat{V}^{\mathbf{pess}}_{t+1})(x,a) \leq (\B_{t}^F V^*_{t+1})(x,a).
\end{align*} The last inequality holds  with probability at least $1-c\sum_{k=t}^T\frac{|\mathcal{A}_k^{\mathcal{D}_N}|}{N}$. The proof is completed by applying the union bound for all $t$.
\end{proof}

\subsection{Proof of Lemma \ref{lemma:decomp}} \label{appendix:decomp}

\begin{proof}

Recall that $\mu^{\widehat \pi}=\E\left[V_1^*(X_1)-V_1^{\widehat \pi}(X_1)\right]$, and the aim is to show that \begin{align*}
        \mu^{\widehat \pi} = \underbrace{\sum_{t=1}^T \E^{\pi^*}[ l_t(X_t,A_t) ]}_{J_1}- \underbrace{\sum_{t=1}^T \E^{\widehat \pi}[ l_t(X_t,A_t) ]}_{J_2} +\underbrace{\sum_{t=1}^T \E^{\pi^*}\left[ \sum_{a\in \mathcal A} \widehat Q_t(X_t,a)\left(\pi^*_t(a|X_t)-\widehat\pi_t(a|X_t)\right) \right]}_{J_3},
    \end{align*} where $l_t(x,a)=( \B_t \widehat V_{t+1})(x,a)- \widehat Q_{t}(x,a).$ Note that the expectations $\E^{\pi^*}$ and $\E^{\widehat \pi}$ are taken with respect to the trajectories induced by the optimal policy and the estimated policy, respectively, with the functions $\widehat V_t(\cdot)$ and $\widehat Q_t(\cdot,\cdot)$ 
    being held fixed.
    
    By the definition of the Bellman operator, we have \begin{align*} 
       &\E^{\pi^*}[ l_t(X_t,A_t) ]=  \E^{\pi^*}\left[\E\left(R_t+\widehat V_{t+1}(X_{t+1})|X_t,A_t \right)-\widehat Q_{t}(X_t,A_t) \right] \\
      = \; & \E^{\pi^*}\left(R_t \right)+\E^{\pi^*}\left[ \widehat V_{t+1}(X_{t+1}) \right]-\E^{\pi^*}\left[\widehat Q_{t}(X_t,A_t) \right],
    \end{align*}  where the last equality follows the law of iterated expectations. Thus, we have \begin{align*} 
       &\E^{\pi^*}[ l_t(X_t,A_t) ]+\E^{\pi^*}[ l_{t+1}(X_{t+1},A_{t+1}) ] \\
      = \; & \E^{\pi^*}\left(R_t \right)-\E^{\pi^*}\left[\widehat Q_{t}(X_t,A_t) \right]+\E^{\pi^*}\left(R_{t+1} \right)+\E^{\pi^*}\left[ \widehat V_{t+2}(X_{t+2}) \right]\\
      &\; + \underbrace{\E^{\pi^*}\left[ \widehat V_{t+1}(X_{t+1}) \right]-\E^{\pi^*}\left[\widehat Q_{t+1}(X_{t+1},A_{t+1}) \right]}_{=\E^{\pi^*}\left[\sum_{a\in \mathcal A} \widehat Q_{t+1}(X_{t+1},a)\left(\widehat\pi_t(a|X_{t+1})-\pi^*_t(a|X_{t+1})\right)\right]}.
    \end{align*}
Thus, term $J_1$ is equivalent to \begin{align*}
    &\sum_{t=1}^T \E^{\pi^*}[ l_t(X_t,A_t) ]\\
    =\; & \underbrace{\sum_{t=1}^T \E^{\pi^*}\left(R_t \right)}_{=\E(V_1^*(X_1))}+\sum_{t=2}^T \E^{\pi^*}\left[ \sum_{a\in \mathcal A} \widehat Q_t(X_t,a)\left(\widehat\pi_t(a|X_t)-\pi^*_t(a|X_t)\right) \right]-\E^{\pi^*}\left[\widehat Q_{1}(X_1,A_1) \right].
\end{align*} The term $J_2$ can be analogously showed to be \begin{align*}
    &\sum_{t=1}^T \E^{\widehat\pi}[ l_t(X_t,A_t) ]\\
    =\; & \underbrace{\sum_{t=1}^T \E^{\widehat\pi}\left(R_t \right)}_{\E(V_1^{\widehat\pi}(X_1))}+\underbrace{\sum_{t=2}^T \E^{\widehat\pi}\left[ \sum_{a\in \mathcal A} \widehat Q_t(X_t,a)\left(\widehat\pi_t(a|X_t)-\widehat\pi_t(a|X_t)\right) \right]}_{0}-\E^{\widehat\pi}\left[\widehat Q_{1}(X_1,A_1) \right].
\end{align*} Thus $J_1-J_2$ is equal to \begin{align*}
    & \underbrace{\E(V_1^{*}(X_1))-\E(V_1^{\widehat\pi}(X_1))}_{=\mu^{\widehat \pi}}+\sum_{t=2}^T \E^{\pi^*}\left[ \sum_{a\in \mathcal A} \widehat Q_t(X_t,a)\left(\widehat\pi_t(a|X_t)-\pi^*_t(a|X_t)\right) \right]\\
    +\;& \underbrace{\E^{\widehat\pi}\left[\widehat Q_{1}(X_1,A_1) \right]-\E^{\pi^*}\left[\widehat Q_{1}(X_1,A_1) \right]}_{=\E^{\pi^*}\left[ \sum_{a\in \mathcal A} \widehat Q_1(X_1,a)\left(\widehat\pi_1(a|X_1)-\pi^*_1(a|X_1)\right) \right]}\\
    =\; & \mu^{\widehat \pi}+\sum_{t=1}^T \E^{\pi^*}\left[ \sum_{a\in \mathcal A} \widehat Q_t(X_t,a)\left(\widehat\pi_t(a|X_t)-\pi^*_t(a|X_t)\right) \right].
\end{align*} Finally, we have $J_1-J_2+J_3=\mu^{\widehat \pi}$, this completes the proof.
\end{proof}

\subsection{Lemma \ref{lemma:qerror} and its Proof} \label{appendix:lemma2}

\begin{lemma} \label{lemma:qerror}
For any $t,$ $x \in \mathcal{X}$, $a \in \mathcal A $, and  $F \in \Omega_t(x,a) $,
    $$\left|(\B_t \widehat V_{t+1})(x,a)-\widehat Q_{t}(x,a;F)\right| \leq c|F_t(d|a)-F^{(d)}| .$$
\end{lemma}

\begin{proof}
    
Following Proposition \ref{prop: bellman}, under the true Bellman operator, we have \begin{align*}
        (\B_t \widehat V_{t+1})(x,a)= a x 
    + \sum_{d=0}^{x-1}
    \bigl[
        \widehat V_{t+1}(x-d)
        - \widehat V_{t+1}(x-d-1)
        - a
    \bigr] F_t(d|a)
    \end{align*}  Then for any $F \in \Omega_t(x,a) $, and 
\begin{align*}
   & (\B_t \widehat V_{t+1})(x,a)-\widehat Q_{t}(x,a;F ) =\sum_{d=0}^{x-1}
    \bigl[
        \widehat V_{t+1}(x-d)
        - \widehat V_{t+1}(x-d-1)
        - a
    \bigr] (F_t(d|a)-F^{(d)})\\
   \leq & \; \max\{a_{\max}-a, a-a_{\min}\} \sum_{d=0}^{x-1} |F_t(d|a)-F^{(d)}|
\end{align*}

\end{proof}

\subsection{Proof of Theorem \ref{thm: regret 2}} \label{appendix:regret2}

\begin{proof} Since the theorem concerns the regret bound of the pessimistic approach, throughout  this proof, all the estimated value functions and Q-functions follow the estimated pessimistic policy $\widehat \pi ^{\mathbf{pess}}$.

We first show that the term $J_2$
in Lemma \ref{lemma:decomp} is non-negative. By the union bound, we have with  probability at least $1-\sum_{t=1}^T \frac{|\mathcal{A}_t^{\mathcal{D}_N}|}{N}$, the true demand CDF $F_t(d|a)$ lies in the interval $\Omega_t(x,a)$, uniformly over $t$ and $d \leq x-1$. 
Since $\widehat{Q}_T^{\mathbf{pess}}(x,a)=\min_{F \in \Omega_t(x,a)} (\B_T^F \widehat{V}_{T+1})(x,a)$, we have
    \begin{align*}
    l_t(x,a)=(\B_t \widehat V_{t+1}^{\mathbf{pess}})(x,a)-\widehat Q_{t}^{\mathbf{pess}}(x,a) \geq 0,
\end{align*}  with the same probability, uniformly in $t$.

Since the estimated policy is greedy with respect to the estimated Q-function $\widehat Q_{t}^{\mathbf{pess}}(x,a)$, we have $J_3 \leq 0$.
Thus $\mu^{\mathbf{pess}}(x)$ is upper bounded by $J_1$. To bound $J_1$, we next derive an upper bound for 
$l_t(x,a)$. Using the same technique in the proof of Lemma \ref{lemma:qerror} in Section \ref{appendix:lemma2}, $l_t(x,a)$ is 
of the same order of $\widehat  F_t^U(a)-\widehat  F_t^L(a)$, hence we have
\begin{align*}
    & l_t(x,a) \leq c [ \widehat  F_t^U(a)-\widehat  F_t^L(a)],
\end{align*}
for some positive constant $c$.

In the following, we only consider the case that both $\widehat a_t^U$ and $\widehat a_t^L$ exist, i.e., the confidence interval $\Omega_t(x,a)$ can be constructed using other prices or $a$ itself. We can tackle similarly the other two cases: (a) when $\widehat{a}_t^U$ exists but $\widehat{a}_t^L$ does not, and (b) when $\widehat{a}_t^L$ exists but $\widehat{a}_t^U$ does not. We thus omit the details to save the space.
By definition, it follows that \begin{align*}
    & \widehat  F_t^U(d|a)-\widehat  F_t^L(d|a)=\widehat F_t(d|\widehat a_t^U) +\delta_t(\widehat a_t^U)-\widehat F_t(d|\widehat a_t^L) +\delta_t(\widehat a_t^L)\\
    \leq \; & \widehat F_t( d|a_t^U) +\delta_t( a_t^U)-\widehat F_t( d|a_t^L) +\delta_t( a_t^L) \\
    \leq \; &  F_t(d| a_t^U) +2\delta_t( a_t^U)- F_t( d|a_t^L) +2\delta_t( a_t^L),
\end{align*} 
where the first inequality holds due to the observation that $\widehat a_t^U$ and $\widehat a_t^L$ are the prices which provide the tightest interval; and the second inequality follows from concentration inequality, with the same probability, uniformly in $t$. 

Now let  \begin{align*}
   \eta_t(x,a) \equiv 2(x-1)\left[\delta_t( a_t^U)+\delta_t( a_t^L)\right]+\sum_d^{x-1} \left[F_t(d| a_t^U) - F_t( d|a_t^L)\right],
\end{align*} 

Together with Lemma \ref{lemma:decomp}, this yields the result that with probability at least $1-\sum_{t=1}^T\sum_{a\in\mathcal A _t^{\mathcal{D}_N}}  \frac{1}{N_t(a)}$, \begin{align*}
    &\mu^{\mathbf{pess}}(x) \leq  \sum_{t=1}^T \E^{\pi^*}[ l_t(X_t,A_t) ] \leq \sum_{t=1}^T \E^{\pi^*}\left( \eta_t (X_t,A_t)  \right)\\
    =& \;\sum_{t=1}^T \E^{\pi^*}\left[ \eta_t (X_t,A_t) \mathds{1}(A_t \in \mathcal{M}_t(X_t)) \right]+\sum_{t=1}^T \E^{\pi^*}\left[ \eta_t (X_t,A_t) \mathds{1}(A_t \notin \mathcal{M}_t(X_t)) \right]\\ 
    =& \; \sum_{t=1}^T \E^{\pi^b}\left[  \frac{\prob_t^{\pi^*}(X_t, A_t)}{\prob_t^{\pi^b}(X_t,A_t)} \eta_t (X_t,A_t) \mathds{1}(A_t \in \mathcal{M}_t(X_t))  \right]+ \sum_{t=1}^T\E^{\pi^*}\left( \eta_t (X_t,A_t)\mathds{1}(A_t \notin \mathcal{M}_t(X_t)) \right)\\
    =& \; \underbrace{\sum_{t=1}^T \E^{\pi^b}\left[  \frac{\prob_t^{\pi^*}(X_t, A_t)}{\prob_t^{\pi^b}(X_t,A_t)} \eta_t (X_t,A_t) \mathds{1}(A_t \in \mathcal{M}_t(X_t), A_t \in \mathcal{A}_t^{\mathcal{D}_N} )  \right]}_{H_1}\\
    & +\underbrace{\sum_{t=1}^T \E^{\pi^b}\left[  \frac{\prob_t^{\pi^*}(X_t, A_t)}{\prob_t^{\pi^b}(X_t,A_t)} \eta_t (X_t,A_t) \mathds{1}(A_t \in \mathcal{M}_t(X_t), A_t \notin \mathcal{A}_t^{\mathcal{D}_N} )  \right]}_{H_2}\\
    & + \underbrace{\sum_{t=1}^T\E^{\pi^*}\left( \eta_t (X_t,A_t)\mathds{1}(A_t \notin \mathcal{M}_t(X_t)) \right)}_{H_3}.
\end{align*} The second equality holds as on the event $\{A_t \in \mathcal{M}_t(X_t) \}$, $\prob_t^{\pi^*}(X_t, A_t)/\prob_t^{\pi^b}(X_t,A_t)$ is bounded from above. This allows us to apply the change of measure to replace $\E^{\pi^*}$ with $\E^{\pi^b}$.

We now first analyze $H_1$, since $A_t \in \mathcal{A}_t^{\mathcal{D}_N}$, by definition, we have $A_t^U=A_t^L=A_t$, and hence $\eta_t(X_t,A_t)$ reduces to $4(X_t-1)\delta_t(A_t)$. Therefore, we have \begin{gather*}
    H_1 = \sum_{t=1}^T 4\E^{\pi^b}\left[  (X_t-1)\frac{\prob_t^{\pi^*}(X_t, A_t)}{\prob_t^{\pi^b}(X_t,A_t)} \sqrt{\log N/N_t(A_t) } \mathds{1}(A_t \in \mathcal{M}_t(X_t), A_t \in \mathcal{A}_t^{\mathcal{D}_N} )  \right]\\
    \leq  \sum_{t=1}^T 4 \E^{\pi^b}\left[ (X_t-1) \frac{\prob_t^{\pi^*}(X_t, A_t)}{\prob_t^{\pi^b}(X_t,A_t)} \sqrt{\log N/N_t(A_t) } \mathds{1}(A_t \in \mathcal{M}_t(X_t))  \right].
\end{gather*}

We next focus on $H_2$. First, notice that with probability at least $1-\kappa_t$, $\mathcal{A}_t^{\pi^b}=\mathcal{A}_t^{\mathcal{D}_N}$. Thus, conditional on the event $\mathcal{A}_t^{\pi^b}=\mathcal{A}_t^{\mathcal{D}_N}$, if $A_t \in \mathcal{M}_t(X_t)$, but $A_t \notin \mathcal{A}_t^{\mathcal{D}_N}$, this implies that the ratio $\frac{\prob_t^{\pi^*}(X_t, A_t)}{\prob_t^{\pi^b}(X_t,A_t)}=0$. It follows that $H_2=0$ with probability at least $1-\kappa_t$. Thus, so far we have shown that \begin{align*}
    \mu^{\mathbf{pess}}(x) \leq & \sum_{t=1}^T 4\E^{\pi^b}\left[ (X_t-1)  \frac{\prob_t^{\pi^*}(X_t, A_t)}{\prob_t^{\pi^b}(X_t,A_t)} \sqrt{\log N/N_t(A_t) } \mathds{1}(A_t \in \mathcal{M}_t(x))  \right]\\
     &\;  + \sum_{t=1}^T\E^{\pi^*}\left( \eta_t (X_t,A_t)\mathds{1}(A_t \notin \mathcal{M}_t(x)) \right).
\end{align*}

Finally, we show that this upper bound can be achieved. Since the action and state spaces in our setting are discrete, it can be formulated as the linear MDP \citep{jin2021pessimism}. The first part $H_1$ achieves the minimax optimal 
rate for offline linear MDP 
up to multiplicative factors of the dimension and horizon,
following Theorems 4.4, 4.6, and Corollary 4.5 for linear MDP in \citet{jin2021pessimism}. Specifically, Theorem 4.4 established the upper bound for the linear MDP regret, while Theorem 4.6 provided the matching lower bound. Additionally, Corollary 4.5 demonstrated that the regret converges at a rate of $N^{-1/2}$, up to logarithmic factors, i.e., $\sqrt{\log N / N}$.

It now remains to analyze the third term $H_3$, which corresponds to the error due to the unobserved optimal price. In what follows, we only deal with the scenario that both $a^U$ and $a^L$ exist so that
\begin{align*}
    \eta_t (x,a) =  \sum_d^{x-1} \left[\underbrace{F_t(d|a_t^u) -F_t(d|a_t^l)}_{L_1}\right]+2(x-1)\left[\underbrace{\delta_t(a_t^u)+\delta_t(a_t^l)}_{L_2}\right],
\end{align*} since the other two cases can be handled analogously. The term $L_1$ is the length of the interval $[ F_t(d|a_t^l), F_t(d|a_t^u)]$ for $F_t(d|a)$. 
Conditional on the event $\{\mathcal{A}_t^{\pi^b}=\mathcal{A}_t^{\mathcal{D}_N}\}$, the interval $[ F_t(d|a_t^l), F_t(d|a_t^u)]$ is the sharpest interval, as illustrated in the main paper. In other words, any value in the interval $[ F_t(d|a_t^l), F_t(d|a_t^u)]$ cannot be excluded as the potential true value of parameter. Thus, $F_t(d|a)$ can be either $ F_t(d|a_t^l)$ or $ F_t(d|a_t^u)$. Consequently, term $L_1$ is attainable. Combining this result with the fact that the second term $L_2$ arises from approximating the expectation with the empirical mean, which is rate-optimal, we complete the proof.
\end{proof}

\newpage

\subsection{ Proof of Theorem \ref{theorem:regret3}} \label{appendix:mm regret}

\begin{proof} Since this proof focuses on the regret bound of the opportunistic approach. Throughout, all the estimated value functions and the estimated Q-functions follows the estimated opportunistic policy $\widehat \pi ^{\mathbf{opp}}$.

The proof consists of three parts. We first show that $l_t(x,a)$ is a non-negative term; then we show that the term $J_3$ in Lemma \ref{lemma:decomp} is non-positive. Finally, we upper bound the term $J_1$ in Lemma \ref{lemma:decomp}.

We now show that $l_t(x,a)$ is a non-negative term. Recall that for for any $a \in \mathcal{A}$ and any $t$, the opportunistic approach first seek the value of $F \in \Omega_t(x,a), F'\in \Omega_t(x, a')$ and $a' \in \mathcal{ A}$ to maximize the regret, i.e., \begin{align*} 
   \max_{F' \in \Omega_t(x, a'), F \in \Omega_t(x,a)}\left [\max_{a'}\widehat Q_t^{\mathbf{opp}}(x,a';F')-\widehat Q_t^{\mathbf{opp}}(x,a;F)\right].
\end{align*} Thus we have \begin{align*}
    \max_{a'\in \mathcal{A}} \widehat Q_t^{\mathbf{opp}} (x,a'; F_t^{\mathbf{opp}}(a^+))- \widehat Q_t ^{\mathbf{opp}} (x,a; F_t^{\mathbf{opp}}(a)) \geq \max_{a'\in\mathcal{ A}} \widehat Q_t^{\mathbf{opp}} (x,a'; F_t^{\mathbf{opp}}(a^+))- \widehat Q_t ^{\mathbf{opp}} (x,a;F_t(d|a)),
\end{align*} since the pair $( F_t^{\mathbf{opp}}(a^+),  F_t^{\mathbf{opp}}(a))$ is the maximizer of the corresponding optimization problem, for a fixed $a$. It then follows that \begin{align*}
    \widehat Q_t ^{\mathbf{opp}} (x,a; F_t^{\mathbf{opp}}(a)) \leq \widehat Q_t ^{\mathbf{opp}} (x,a,F_t(d|a))= (\B_t \widehat V_{t+1}^{\mathbf{opp}})(x,a).
\end{align*} Recall that $l_t(x,a)=(\B_t \widehat V_{t+1}^{\mathbf{opp}})(x,a)-\widehat Q_t ^{\mathbf{opp}} (x,a; F_t^{\mathbf{opp}}(a))$ for the opportunistic approach. Thus we have proved that $l_t(x,a) \geq 0$ for any $x$ and $a$.

Now we show that $J_3\leq 0$. For $a_t^*$, its estimated maximum regret is \begin{gather*}
    \widehat Q_t^{\mathbf{opp}} (x,a_m(a_t^*); F_t^{\mathbf{opp}}(a_m(a_t^*)))-\widehat Q_t ^{\mathbf{opp}} (x,a_t^*; F_t^{\mathbf{opp}}(a_t^*)),
\end{gather*} where $a_m(a_t^*) \in \argmax_{a'  \in \mathcal{A}} \left[\widehat Q_t^{\mathbf{opp}} (x,a', F_t^{\mathbf{opp}}(a^+))- \widehat Q_t ^{\mathbf{opp}} (x,a_t^*; F_t^{\mathbf{opp}}(a_t^*))\right]$. 

Similarly, the estimated maximum regret for $\widehat a_t^{\mathbf{opp}} $ is  \begin{gather*}
    \widehat Q_t^{\mathbf{opp} } (x,a_m(\widehat a_t^{\mathbf{opp}}); F_t^{\mathbf{opp}}(a_m(\widehat a_t^{\mathbf{opp}})))-\widehat Q_t ^{ \mathbf{opp}} (x,\widehat a_t^{\mathbf{opp}}; F_t^{\mathbf{opp}}(\widehat a_t^{\mathbf{opp}})).
\end{gather*} We first consider the scenario that $a_m(\widehat a_t^{\mathbf{opp}}) \neq \widehat a_t^{\mathbf{opp}}$. Under this scenario, the maximization problem over $\lambda$'s is separable, thus we have 
\begin{align} \label{eq:opp max}
    \notag & \widehat Q_t^{\mathbf{opp}} (x,a_m(a_t^*); F_t^{\mathbf{opp}}(a_m(a_t^*)))-\widehat Q_t ^{\mathbf{opp}} (x,\widehat a_t^{\mathbf{opp}}; F_t^{\mathbf{opp}}(\widehat a_t^{\mathbf{opp}}))\\
    \leq & \; \widehat Q_t^{\mathbf{opp}} (x,a_m(\widehat a_t^{\mathbf{opp}}); F_t^{\mathbf{opp}}(a_m(\widehat a_t^{\mathbf{opp}})))-\widehat Q_t ^{\mathbf{opp}} (x,\widehat a_t^{\mathbf{opp}}; F_t^{\mathbf{opp}}(\widehat a_t^{\mathbf{opp}})).
\end{align} The inequality follows from the fact that the pair $\left[ F_t^{\mathbf{opp}}(a_m(\widehat a_t^{\mathbf{opp}})), a_m(\widehat a_t^{\mathbf{opp}})\right]$ is the maximizer of the regret given the fixed pair $\left[ F_t^{\mathbf{opp}}(\widehat a_t^{\mathbf{opp}}),\widehat a_t^{\mathbf{opp}}\right]$.

We next consider the case that $a_m(\widehat a_t^{\mathbf{opp}}) = \widehat a_t^{\mathbf{opp}}$.  This setting happens only if the action $\widehat a_t^{\mathbf{opp}}$ is so powerful that its worst performance surpasses the best performance of any other action. Thus we have, 
\begin{align*}
    & \max_{F' \in \Omega_t(x, a')} \max_{a' \neq \widehat a_t^{\mathbf{opp}}}\widehat Q_t^{\mathbf{opp}} (x,a';F')-\widehat Q_t ^{\mathbf{opp}} (x,\widehat a_t^{\mathbf{opp}}; F_t^{\mathbf{opp}}(\widehat a_t^{\mathbf{opp}})) \\
    \leq & \; 0=  \widehat Q_t^{\mathbf{opp}} (x,a_m(\widehat a_t^{\mathbf{opp}}); F_t^{\mathbf{opp}}(a_m(\widehat a_t^{\mathbf{opp}})))-\widehat Q_t ^{\mathbf{opp}} (x,\widehat a_t^{\mathbf{opp}}; F_t^{\mathbf{opp}}(\widehat a_t^{\mathbf{opp}})).
\end{align*}

In either case, Equation \eqref{eq:opp max} holds. Thus we have \begin{align*}
    & \widehat Q_t^{\mathbf{opp}} (x,a_m(a_t^*); F_t^{\mathbf{opp}}(a_m(a_t^*)))-\widehat Q_t ^{\mathbf{opp}} (x,\widehat a_t^{\mathbf{opp}}; F_t^{\mathbf{opp}}(\widehat a_t^{\mathbf{opp}}))\\
    \leq & \; \widehat Q_t^{\mathbf{opp}} (x,a_m(\widehat a_t^{\mathbf{opp}}); F_t^{\mathbf{opp}}(a_m(\widehat a_t^{\mathbf{opp}})))-\widehat Q_t ^{\mathbf{opp}} (x,\widehat a_t^{\mathbf{opp}}; F_t^{\mathbf{opp}}(\widehat a_t^{\mathbf{opp}})) \\
    \leq & \; \widehat Q_t^{\mathbf{opp}} (x,a_m(a_t^*); F_t^{\mathbf{opp}}(a_m(a_t^*)))-\widehat Q_t ^{\mathbf{opp}} (x,a_t^*; F_t^{\mathbf{opp}}(a_t^*)).
\end{align*} The first inequality follows from Equation \eqref{eq:opp max}, while the second inequality is due to $\widehat a_t^{\mathbf{opp}}$ is chosen such that the maximum regret is minimum (Equation \eqref{eq: minimax}). Thus we have \begin{align} \label{eq:opp greedy}
    \notag & \widehat Q_t^{\mathbf{opp}} (x,a_m(a_t^*); F_t^{\mathbf{opp}}(a_m(a_t^*)))-\widehat Q_t ^{\mathbf{opp}} (x,\widehat a_t^{\mathbf{opp}}; F_t^{\mathbf{opp}}(\widehat a_t^{\mathbf{opp}})) \\
    \notag \leq & \; \widehat Q_t^{\mathbf{opp}} (x,a_m(a_t^*); F_t^{\mathbf{opp}}(a_m(a_t^*)))-\widehat Q_t ^{\mathbf{opp}} (x,a_t^*; F_t^{\mathbf{opp}}(a_t^*)), \\
    \implies & \widehat Q_t ^{\mathbf{opp}} (x,\widehat a_t^{\mathbf{opp}}; F_t^{\mathbf{opp}}(\widehat a_t^{\mathbf{opp}})) 
    \geq \; \widehat Q_t ^{\mathbf{opp}} (x,a_t^*; F_t^{\mathbf{opp}}(a_t^*)). 
\end{align} Recall that $J_3=\sum_{t=1}^T \E^{\pi^*}\left[ 
    \widehat Q_t ^{\mathbf{opp}} (x,a_t^*; F_t^{\mathbf{opp}}(a_t^*))-\widehat Q_t ^{\mathbf{opp}} (x,\widehat a_t^{\mathbf{opp}}; F_t^{\mathbf{opp}}(\widehat a_t^{\mathbf{opp}})) \right]$ for the opportunistic approach, which is non-positive according to Equation \eqref{eq:opp greedy}.

To summarize, we have so far shown that (i) $l_t(x,a)\geq 0$; (ii) and the term $J_3$ is non-positive. It now remains to upper bound $l_t(x,a)$, using the same technique in the proof of Lemma \ref{lemma:qerror} in Section \ref{appendix:lemma2} and Theorem \ref{thm: regret 2} ,
\begin{align*}
    & l_t(x,a) \leq c [ \widehat  F_t^U(a)-\widehat  F_t^L(a)].
\end{align*} 
It thus remains to upper bound $\widehat  F_t^U(a)-\widehat  F_t^L(a)$, and show that it is rate optimal, which can be analogously demonstrated using the method employed in the proof of Theorem \ref{thm: regret 2}. We omit the details to save the space. This completes the proof.
\end{proof}

\end{document}